%% file: limits-dropout.tex
\documentclass[11pt]{article}
\usepackage[margin=1.1in,a4paper]{geometry}

\usepackage{mathtools}
\usepackage{amsmath, amssymb, bbm, natbib, graphicx, url, amsthm, subcaption, float, dsfont}
\usepackage{mathrsfs}
\usepackage{algorithm, algpseudocode}
\usepackage[mathcal]{eucal}
\usepackage{tikz}
\usetikzlibrary{calc,arrows.meta,positioning, decorations.pathreplacing}
\usepackage{verbatim}
\usepackage{bm}
 \usepackage{tcolorbox}
 \newtcolorbox{assbox}{colback=black!5!white,colframe=black!75!black}
  \newtcolorbox{thmbox}{colback=red!5!white,colframe=red!75!black}

\usepackage{booktabs}       % professional-quality tables
\usepackage{amsfonts}       % blackboard math symbols
\usepackage{enumitem}
\usepackage{microtype}      % microtypography
\usepackage{xcolor}
\usepackage[colorlinks=true,citecolor=blue]{hyperref}    % hyperlinks
\usepackage{diagbox}   
\usepackage{tcolorbox}
\usepackage{nicefrac}

\input{shortcuts}

\definecolor{ForestGreen}{cmyk}{0.91,0,0.88,0.12}
\colorlet{pierrem}{ForestGreen}

\graphicspath{ {./images/} } 

\title{
Phase Diagram of Dropout for Two-Layer Neural Networks in the Mean-Field Regime
 }

\newcommand*\samethanks[1][\value{footnote}]{\footnotemark[#1]}
\author{
L\'ena\"ic Chizat\textsuperscript{1}\thanks{ Correspondence: \texttt{lenaic.chizat@epfl.ch}.} ,\quad
Pierre Marion\textsuperscript{2}\thanks{Work mostly done while at the Institute of Mathematics, EPFL.} ,\quad
Yerkin Yesbay\textsuperscript{3}\samethanks
}
\date{\small
\textsuperscript{1} École Polytechnique Fédérale de Lausanne (EPFL), Institute of Mathematics, Lausanne, Switzerland \\
\textsuperscript{2}Inria, École Normale Supérieure, PSL Research University, Paris, France \\
\textsuperscript{3}Télécom Paris, Institut Polytechnique de Paris, Palaiseau, France \\
}

\begin{document}
\maketitle
\begin{abstract}
Dropout is a standard training technique for neural networks that consists of randomly deactivating units at each step of their gradient-based training. It is known to improve performance in many settings, including in the large-scale training of language or vision models.
As a first step towards understanding the role of dropout in large neural networks, we study the large-width asymptotics of gradient descent with dropout on two-layer neural networks with the mean-field initialization scale. We obtain a rich asymptotic phase diagram that exhibits five distinct nondegenerate phases depending on the relative magnitudes of the dropout rate, the learning rate, and the width. Notably, we find that the well-studied ``penalty'' effect of dropout only persists in the limit with impractically small learning rates of order $O(1/\text{width})$. For larger learning rates, this effect disappears and in the limit, dropout is equivalent to  a ``random geometry'' technique, where the gradients are thinned randomly  \emph{after} the forward and backward pass have been computed. In this asymptotic regime, the limit is described by a mean-field jump process where the neurons' update times follow independent Poisson or Bernoulli clocks (depending on whether the learning rate vanishes or not). For {some of} the phases, we obtain a description of the limit dynamics both in path-space and in distribution-space. The convergence proofs involve a mix of tools from mean-field particle systems and stochastic processes. Together, our results lay the groundwork for a renewed theoretical understanding of dropout in large-scale neural networks.
\end{abstract}

\section{Introduction} \label{sec:intro}
Stochastic gradient descent (SGD) is the workhorse of deep learning optimization, but its raw form is rarely used in practice. Over the years, researchers have proposed many training techniques—adaptive learning rates, gradient clipping, weight decay, and more—on top of SGD to accelerate convergence or improve test performance. Among these heuristics, \emph{dropout}~\citep{srivastava2014dropout} has stood the test of time because it is easy to implement, architecture-agnostic, and the associated performance improvement is empirically robust across tasks. It consists in randomly deactivating a subset of units/neurons during each forward–backward pass. This random \emph{thinning} of the network induces a structured, non-centered, perturbation of the gradients whose theoretical analysis remains challenging. 

The original intuition for dropout's performance is that ``it prevents co-adaptation between units'' and that it approximates the ``combined opinions of exponentially many dropout nets''~\citep{srivastava2014dropout}. However, these intuitions remain hypothetical and {the recent theoretical literature on dropout follows a quite different route, which we now outline}.

Let $\theta$ be the parameters of a neural network (NN), $\Ll(\theta)$ the training loss and $\Ll^\eta(\theta)$ the training loss perturbed with a dropout mask encoded by the random variable $\eta$.
Following e.g.~\cite{wei2020implicit}, one can decompose the noise induced by dropout on gradients as
\begin{align}\label{eq:bias-variance}
\nabla \Ll^\eta(\theta) -  \nabla \Ll(\theta)&= \underbrace{\E_\eta [\nabla \Ll^\eta(\theta)]-\nabla \Ll(\theta)}_{\text{bias}}+\underbrace{\nabla \Ll^\eta(\theta)- \E_\eta [\nabla \Ll^\eta(\theta)]}_{\text{centered noise}}.
%&= \nabla (\Ll+P)(\theta) + \nabla \Ll^\eta(\theta)- \E_\eta [\nabla \Ll^\eta(\theta)]
\end{align}
Exchanging gradient and expectation, we notice that the bias is of the form $ \nabla P(\theta)$ with 
\begin{align}\label{eq:penalization-def}
P(\theta) = \E_\eta[ \Ll^\eta(\theta)]- \Ll(\theta).
\end{align}
This function, known as \emph{dropout penalty}, can be made explicit in simple settings, such as linearly-parameterized models or two-layer NNs with squared loss~\citep[][{Proposition 4, and references therein}]{arora2021dropout}. More generally, an informal but empirically accurate approximation of the penalty can be given in terms of the layer-to-layer Jacobians~\citep{wei2020implicit}. Much of the theoretical literature on dropout studies the optimization and generalization properties of this penalty, in particular by investigating its connection with more classical regularization terms such as the weight $\ell_2$ norm or path norm \citep[notably,][]{wager2013dropout,helmbold2015inductive,mou2018dropout,arora2021dropout}. As for the stochastic term in~\eqref{eq:bias-variance}, its effect has been studied in~\citet{mianjy2020convergence,wei2020implicit,zhang2024implicit,zhang2024stochastic}.

Despite this extensive literature, the question of the scaling of dropout to large-width networks has received little attention. This is a particularly important issue given that dropout is still used by practitioners for large-scale models~\citep{vaswani2017attention,radford2018improving,devlin2019bert,dosovitskiy2021an,hu2021lora,ouyang2022training,ramesh2022hierarchical,taylor2022galactica}. This begs the question of revisiting the theory of dropout in the context of large-scale networks, which is our focus in this paper.  To the best of our knowledge, we are the first work to study the scaling of dropout for large NNs from a theoretical angle. A couple of experimental papers investigated specifically the question of the role of dropout in large NNs for language and vision tasks~\citep{liu2023dropout,xue2023torepeat}. These studies confirm the effectiveness of dropout in these settings.

This question falls into the broader effort of understanding the scaling properties of large-scale NNs. In the absence of dropout, prior work has shown that the training dynamics of large NNs already exhibit a rich asymptotic phase diagram~\citep{chizat2019lazy}, with in particular the so-called neural tangent kernel (NTK) regime \citep{jacot2018neural}, where (overparameterized) models may reach zero training loss albeit without building data-dependent representations, and the richer feature learning \citep{chizat2018global,mei2018mean,sirignano2020mean,rotskoff2018neural} associated to the so-called mean-field scaling or $\mu$-parameterization ($\mu$P)~\citep{yang2021tensor}. These scalings were investigated in many works for various architectures~\citep[e.g.][]{geiger2020disentangling, yang2021tensor,  vyas2023feature}. A current active topic in this line of  research concerns the large-depth scaling of ResNets~\cite[see][and references therein]{dey2025don, chizat2025hidden}.

\paragraph{Contributions.} 
This paper is the first step in a broader effort to understand the theoretical properties of dropout in large-scale models. To establish a rigorous foundation, we adopt a mathematical and descriptive approach, focusing on a restricted yet illustrative setting: two-layer networks trained with gradient descent and dropout, using the squared loss and the mean-field scaling. Despite its simplicity, this setup already reveals surprising insights and challenges several widely-held beliefs about dropout.

Our main result (Theorem~\ref{thm:main} {in Section \ref{sec:main-result}}) characterizes all possible scaling limits of the dynamics, depending on {the relative scalings of} the width, learning rate, and dropout rate. {The proof, given in Section \ref{sec:proofs}, relies on different constructions depending on the scaling of parameters.}
  
In Section~\ref{sec:discussions}, we analyze the qualitative properties of these limiting dynamics. In particular:
  \begin{itemize}
  \item All regimes involving the penalization term~\eqref{eq:penalization-def} require a learning rate of order \( O(1/\text{width}) \) which is computationally impractical. {A more efficient alternative that does not require small learning rates is to explicitly add a deterministic penalty $P(\theta)$ to the risk. We show in Proposition~\ref{prop:ram-equivalent} the equivalence of this procedure with  dropout in the small learning rate regime.}

  \item For larger learning rates, we show that the asymptotic effect of dropout is equivalent to gradient descent in a random metric—specifically, a random block-coordinate descent scheme (Proposition~\ref{prop:ram-equivalent}). We investigate the potential advantages of this regime and argue that it allows certain features to take larger steps than the local sharpness of the loss would otherwise permit (see Section~\ref{sec:LLR}). {We also show that a time-averaging of the NTK (Section~\ref{sec:NTK-time}) occurs in this regime.}
  \end{itemize}

Let us now briefly comment on the mathematical techniques used in our proofs. The two main technical contributions correspond to cases~\ref{case:WGF} and~\ref{case:jump} of Theorem~\ref{thm:main}.

In case~\ref{case:WGF}, we establish convergence of the dynamics to a penalized Wasserstein gradient flow. The analysis relies on coupling techniques for mean-field particle systems~\cite[see e.g.][]{mckean1967propagation, sznitman1991topics, lacker2018mean}, as well as martingale methods to control the continuous-time limit of stochastic terms. While similar tools have been combined previously~\cite[e.g.][]{mei2019mean}, our setting is different and requires careful error decomposition, as some of the relevant limits do not commute. We derive this result as a special case of a more general mean-field stochastic approximation theorem (Section~\ref{eq:mean-field-sto-app}), which may be of independent interest. {This approach allows to derive convergence both in path space and in distribution space.}

In case~\ref{case:jump}, we prove convergence to a mean-field jump process. This is handled by constructing an asynchronous coupling between the discrete dynamics—driven by binomial clocks—and the limiting process—driven by a Poisson clock. We also address the well-posedness of the limiting dynamics, both at the SDE and PDE levels. As we were not able to obtain a suitable uniform-in-time particle-level error control, our convergence result is formulated only in terms of time-marginal distributions, and we leave the question of pathwise convergence open for this case.

\section{Setup}
\paragraph{Two-layer neural networks.} Consider a two-layer neural network (2NN) of width $n\in \NN^*$ of the form $f(\theta,z) = \frac1n \sum_{i=1}^n \phi(x^i,z)$ where $\theta=(x^1,\dots,x^n)\in (\RR^p)^n$ are the weights, $z\in \RR^d$ is the input and $\phi:\RR^p\times \RR^d\to \RR$. The typical structure of a 2NN is recovered when ${\phi}(x,z)=a\sigma (b^\top z + c)$ where $\sigma:\RR\to \RR$ is the activation function and $x=(a,b,c)\in \RR\times \RR^d\times \RR$.

Given a training set $(z_i,y_i)_{i=1}^m \in (\RR^d\times \RR)^m$, it will be convenient to introduce the {feature} map $\bm\phi: \RR^p\to \RR^m$ defined by $\bm\phi(x)=(\phi(x,z_1),\dots, \phi(x,z_m))$. Analogously, we define $\bm f(\theta)=(f(\theta,z_1),\dots, f(\theta,z_m))$, which we call the \emph{predictor} parameterized by $\theta$. By construction, it holds
\begin{align*}
\bm f (\theta) = \frac1n \sum_{i=1}^n \bm \phi (x^i),&& \theta=(x^1,\dots,x^n).
\end{align*}

All our theoretical results are proved under the following assumption:

\begin{assumption}[Regularity of $\bm \phi$]\label{ass:phi} The function $\bm \phi:\RR^p\to \RR^m$ is bounded and differentiable with a bounded and Lipschitz differential. We write  $D\bm \phi(x)\in \RR^{m\times p}$ for the Jacobian matrix of $\bm \phi$ at $x$.
\end{assumption}

Our boundedness and regularity assumptions on $\bm \phi$ are quite strong (in particular they {require bounded weights in the case} of standard 2NN presented above) but they allow focusing on the core ideas in the proofs. We believe that it could be possible to relax these assumptions at the cost of an increased level of technicality (see e.g.~\cite{wojtowytsch2020convergence, chizat2020implicit} for the mean-field limit of ReLU 2NN, where $\bm \phi$ satisfies none of these assumptions).
 
 \paragraph{Dropout.} In our context, dropout consists in modifying the predictor at each step of the training process by applying a random mask to the sum defining $\bm f$. The mask is encoded by a vector $\eta=(\eta^1,\dots,\eta^n)\in \RR^n$ with i.i.d.~entries and the \emph{dropout predictor} is
\begin{align*}
\bm{f}^\eta(\theta) = \frac1n \sum_{i=1}^n (1+\eta^i) \bm{\phi}(x^i), &&\text{where}&& \eta^i =\begin{cases}
\frac{1-q}{q} &\text{with probability $q$}\\
-1 &\text{with probability $1-q$}
\end{cases}.
\end{align*}
The \emph{keep} rate $q\in [0,1]$ is a hyperparameter that represents the probability of \emph{not} masking a neuron. It is equal to one minus the dropout rate. Observe that $(1+\eta^i)$ equals $1/q$ with probability $q$ and equals $0$ with probability $1-q$. We will repeatedly use the following identities:
\begin{align*}%\label{eq:eta-moments}
\E [\eta^i] = 0, \quad \E[(\eta^i)^2] = \frac{1-q}{q}, \quad \E[(1+\eta^i)^2] = \frac{1}{q}.
\end{align*}
 \paragraph{GD-dropout.} We consider the squared loss on the training set $(z_i,y_i)_{i=1}^m$. Writing $\bm{y}=(y_1,\dots,y_m)\in \RR^m$,  this leads, in presence of dropout, to the following objective function (which inherits randomness from $\eta$):
 $$
 \Ll^\eta(\theta)  = \frac12 \Vert \bm f^\eta(\theta)-\bm{y}\Vert_2^2.
 $$
 
 The dynamics under investigation, which we refer to as \emph{GD-dropout}, is gradient descent (GD) on the dropout loss with $\eta$ independently resampled at each iteration. Let $\mu_0\in \Pp(\RR^p)$ be an initial probability distribution and $\xi^{1},\xi^2,\dots$  a sequence of independent samples from $\mu_0$. For a given width $n\in \NN^*$, we define the (random) sequence $(\theta^n_k)_{k\in \NN}$ recursively by $\theta^n_0=(\xi^1,\dots,\xi^n)$ and
 \begin{align*}%\label{eq:GD-dropout}
 \theta^n_{k+1} = \theta^n_k - \tau\cdot n\cdot \nabla \Ll^{\eta_k}(\theta^n_k)
 \end{align*}
 where $\tau>0$ is the master learning rate. 
{The factor $n$ multiplying the gradient is introduced so that, in the absence of dropout, the limit is nondegenerate when $n\to \infty$ (this is the mean-field/$\mu$P learning rate, see references in Section \ref{sec:intro}). The dropout mask $\eta_k\in \RR^n$ is independently sampled at each iteration $k$. In this setting, the sources of randomness are thus (i) the random initialization and (ii) the random dropout masks. }

By the chain rule, we obtain the following dynamics in terms of  the weights $\theta^n_k = (X^{1,n}_k,\dots,X^{n,n}_k)$. From now on, we use capital letters for the parameters of the dynamics, to recall that we are dealing with random variables. For $i\in [1:n]$:
 \begin{equation}\label{eq:GD-dropout-weights}
 \left\{
 \begin{aligned}
 X^{i,n}_0 &= \xi^i, \\
 X^{i,n}_{k} &= X^{i,n}_{k-1}  -\tau\cdot (1+\eta^i_k)D\bm{\phi}(X^{i,n}_{k-1})^\top \left(\frac1n \sum_{i=1}^n (1+\eta^i_k) \bm{\phi}(X^{i,n}_{k-1})-\bm{y}\right),\; k\in \NN^*.
 \end{aligned}
 \right.
 \end{equation}
 The rest of the paper is concerned with the analysis of this dynamical system.

 \section{Main result: characterization of the infinite-width limits} \label{sec:main-result}
 There are three hyperparameters (HP) involved in the definition the dynamics: width $n$, keep rate $q$, and learning rate $\tau$. Our goal is to describe how the dynamics behaves in the large-width limit $n\to \infty$ and as a function of  the scaling of other HP with $n$.

\subsection{Notation for infinite-width 2NN} To formulate the limit dynamics, we allow $\bm{f}$ to be parameterized by a probability measure (thereby overloading notations). For $\mu \in \Pp(\RR^p)$, let
\begin{align*}%\label{eq:predictor-measure}
\bm{f} (\mu) = \int_{\RR^p} \bm\phi(x)\d\mu(x).
\end{align*}
A finite-width 2NN with weights $\theta=(x^1,\dots,x^n)$ is obtained as the particular case $\bm{f} (\hat \mu^n)$ where $\hat \mu^n = \frac1n \sum_{i=1}^n\delta_{x^i}$ is the empirical distribution of weights. This parameterization by a measure leads to an objective function $\Ll:\Pp(\RR^p)\to \RR$
$$
\Ll( \mu) = \frac12 \Vert \bm f(\mu)-\bm y\Vert_2^2.
$$
Before stating our main result, let us introduce some useful definitions and notations.
\begin{enumerate}
\item The mean potential $V:\Pp(\RR^p)\times \RR^p\to \RR$ is defined as $V[\mu](x) = \bm\phi(x)^\top (\bm{f}(\mu)-\bm{y})$. This function represents the Fréchet derivative of $\Ll$, seen as a functional over finite signed measures, in the sense that for any finite signed measure $\sigma$  on $\RR^p$, it holds as $\varepsilon \to 0$
$$
\Ll(\mu+\varepsilon \sigma) = \Ll(\mu)+\varepsilon \int V[\mu](x)\d\sigma(x) +o(\varepsilon).
$$
{In the following, we will frequently make use of the gradient of $V$ with respect to its second argument, that is
\[
\nabla V:  \Pp(\RR^p)\times \RR^p\to \RR^p \, , \quad \nabla V[\mu](x) = D \bm\phi(x)^\top (\bm{f}(\mu)-\bm{y}) .
\]
}
\item Given a probability measure $\mu \in \Pp(\RR^p)$ and a measurable map $T:\RR^p\to \RR^p$, the pushforward measure $T_\# \mu$ is defined as $T_\# \mu (B)= \mu(T^{-1}(B))$ for all Borel sets $B\subset \RR^p$. We also denote by $\nabla \cdot $ the divergence operator, which is the opposite of the adjoint of the gradient operator.
\item Given a Polish space $(\Xx,\dist)$ (such as $\RR^p$, see more details in Sec.~\ref{sec:preliminaries}), the $L^1$-Wasserstein distance on the space $\Pp_1(\Xx)$ of probability measures with finite first moment is defined as
$$
W_1(\mu,\nu) = \sup_{\psi\in \text{1-}\Lip(\Xx)} \int \psi \d (\mu-\nu) = \inf_{\pi\in \Pi(\mu,\nu)} \int  \dist(x,y)\d\pi(x,y)
$$
where the supremum runs over $\text{1-}\Lip(\Xx)$ the set of $1$-Lipschitz functions from $\Xx$ to $\RR$, and the infimum runs over the set of transport plans $\Pi(\mu,\nu) = \{ \pi\in \Pp(\Xx\times \Xx)\;;\; \int \pi(\cdot,\d y)= \mu \text{ and } \int \pi(\d x,\cdot) = \nu\}$.
The space $(\Pp_1(\Xx),W_1)$ is a Polish space~\cite[Chap.~6]{villani2008optimal}.
\end{enumerate}

\subsection{Main result} 
In this section, we describe the various limits for dropout-GD in the large-width limit $n\to \infty$. To this end, we consider two positive sequences of learning rates $(\tau_n)$ and keep rates $(q_n)$ indexed by the width $n$. Our main result deals with nondegenerate limits, which are obtained under the following conditions.

\begin{assumption}[Nondegenerate limits]\label{ass:nondegenerate}
Assume that the sequences $\tau_n$, $q_n$,  $\alpha_n = \frac{\tau_n}{q_n}$ and $\beta_n =\frac{1}{nq_n}$ admit \emph{finite} limits, denoted $\tau,q,\alpha$ and $\beta$ respectively. 
\end{assumption}

In addition to the sequences $(\tau_n)$ and $(q_n)$ already introduced, the two new sequences $(\alpha_n)$ and $(\beta_n)$ appearing in this definition are key to capture the asymptotic behavior of the dynamics:
\begin{itemize}
\item The scalar $\alpha_n$ represents the expected time between two updates of the same neuron, as measured in the natural time unit given by $t=k\tau_n$. Indeed, if $\eta^i_k>0$ and $k'=\min\{ \ell>k\;;\; \eta^i_\ell>0\}$ then $k'-k$ is a geometric random variable with expectation $1/q_n$, which multiplied by $\tau_n$ gives $\alpha_n$. The scalar $\alpha_n$ is also equal to the effective learning rate, that is, the scalar factor multiplying the gradient of active units in~\eqref{eq:GD-dropout-weights}. If $\alpha_n\to +\infty$ then the dynamics degenerates in two ways: (i) over a fixed time horizon $[0,T]$, only a vanishing fraction of the neurons are active and (ii) active neurons are subject to a diverging effective learning rate\footnote{The overall effect on the predictor of these two opposite trends -- fewer but larger updates -- depends in general on the growth of $\bm \phi$; but under our assumption that $\bm \phi$ is bounded, the limit dynamics when $\alpha_n\to \infty$ is stationary in predictor space because the effect of the updated units is overall negligible.}. We thus assume that $\alpha = \lim_n \alpha_n <+\infty$.

\item The scalar $\beta_n $ represents the inverse of the expected number of neurons which are updated at a given step. Indeed, this number follows a Bernoulli distribution with parameters $n$ and $q_n$, of expectation $nq_n$. When $\beta_n \to +\infty$, in the limit, almost all the GD steps are trivial with no active neuron\footnote{In this case, we can condition on having at least $1$ neuron active -- which means skipping trivial steps -- which leads, up to time rescaling, to an equivalent dynamics with $\beta\leq1$. Therefore, it would not  be restrictive to assume $\beta\leq 1$ (we only assume $\beta<+\infty$).}. We thus assume $\beta = \lim_n \beta_n<+\infty$.
\end{itemize}
With these key parameters in place, we are ready for our main result, which classifies all possible asymptotic behaviors of the empirical distribution of parameters $\hat \mu^n_k = \frac1n \sum_{i=1}^n \delta_{X^{i,n}_k}$ in the limit $n\to \infty$ (to the exception of the degenerate cases just mentioned). 

This theorem is a corollary of more precise statements given in Section~\ref{sec:proofs} that also describe the asymptotic dynamics of individual neurons. We refer to Figure~\ref{fig:phase-diagram} for a graphical illustration of the phase diagram.

\begin{figure}[t]
\centering
\input{phase-diagram.tex}
\caption{Phase diagram of dropout in two-layer NN with mean-field scaling. For an HP scaling $(n,\tau_n^{-1},q_n^{-1})$, the limit of $(\log n, -\log \tau_n , -\log q_n)/S_n$ where $S_n= \log n -\log \tau_n  -\log q_n$ (when it exists) forms a point in the $2$-simplex, which is represented in the triangle in barycentric coordinates. For instance, the central red point represents proportional limits while the blue line corresponds to $\tau_n^{-1}$ and $q_n^{-1}$ diverging proportionally while $n$ diverges faster. {\color{red!50}Red area:} degenerate limits with either $\alpha=+\infty$ or $\beta=+\infty$. {\color{gray!70}Grey area:} Wasserstein gradient flow limit (the effect of dropout disappears in the limit). {\color{orange}Orange vertex:} discrete-time jump process limit (if $\tau,q>0$). {\color{blue}Blue line:} continuous-time jump process limit.  {\color{purple}Red vertex:} critical limit. {\color{green!70!black}Green line:} penalized Wasserstein gradient flow limit. }
\label{fig:phase-diagram}
\end{figure}
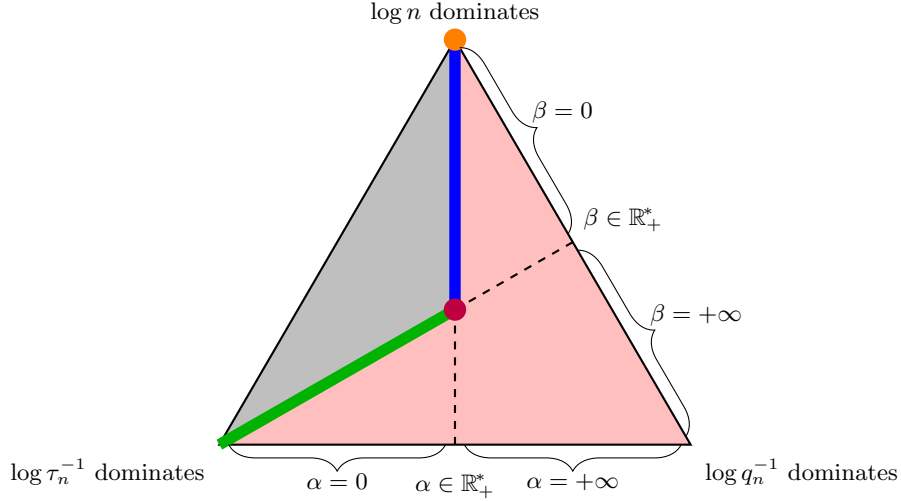

\begin{theorem}[Main theorem]\label{thm:main}
Let Assumption~\ref{ass:phi} and~\ref{ass:nondegenerate} hold. Consider the sequence of GD-dropout dynamics~\eqref{eq:GD-dropout-weights} with $\mu_0\in \Pp_1(\RR^p)$. Then {the limiting dynamics can be classified in four cases depending on the scaling of HP}:
\begin{enumerate}[label=(\Roman*)]
\item \label{case:discrete} \emph{(Discrete-time jump process limit)} If $\tau_n \to \tau>0$ (and hence $q,\alpha>0$ {and $\beta=0$}), then for all $k\in \NN$ it holds $\E [W_1(\hat \mu^n_k,\rho_k)]\to 0$, where the limit sequence $(\rho_k)_{k\in \NN}$ is defined recursively by $\rho_0=\mu_0$ and
\begin{align*}%\label{eq:large-width-limit-1}
 \rho_{k+1} = q \left(\Id-\alpha \nabla V[ \rho_k]\right)_\#  \rho_k + (1-q)\rho_k.
\end{align*}
\item \label{case:WGF} \emph{(Wasserstein Gradient flow limit (penalized if $\beta>0$))} If $\tau_n\to0$ and $\alpha = 0$, then for any $t\geq 0$, $\E [W_1(\hat \mu^n_{\lfloor t/\tau_n \rfloor}, \rho_t)]\to 0$ where $\rho \in \Cc(\RR_+; \Pp_1(\RR^p))$  is the unique weak solution to
\begin{align*}%\label{eq:large-width-limit-2}
\partial_t \rho_t = \nabla \cdot\big(\rho_t (\nabla V[\rho_t]+\beta \nabla P)\big),\qquad \rho_0=\mu_0
\end{align*}
where $P(x)=\frac12 \Vert \bm \phi(x)\Vert_2^2$ is the dropout penalty term.
\item \label{case:jump}\emph{(Continuous-time jump process limit)}  If $\tau_n \to0$, $\alpha > 0$ and $\beta=0$, then for any $t\geq 0$, $\E[W_1(\hat \mu^n_{\lfloor t/\tau_n \rfloor}, \rho_t)]\to 0$ where $\rho \in \Cc(\RR_+; \Pp_1(\RR^p))$  is the unique weak solution to
\begin{align}\label{eq:large-width-limit-3}
\partial_t \rho_t = \alpha^{-1}((\Id-\alpha \nabla V[\rho_t])_\# \rho_t -\rho_t),\qquad \rho_0=\mu_0.
\end{align}
\item \label{case:critical} \emph{(Critical limit)} Otherwise $\tau_n \to0$, $\alpha > 0$ and $\beta>0$. This case is discussed below.%$$
\end{enumerate}
\end{theorem}
\begin{proof}
The cases~\ref{case:discrete},~\ref{case:WGF} and~\ref{case:jump} are corollaries of Proposition~\ref{prop:discrete-limit}, Proposition~\ref{prop:flow-limit} and Proposition~\ref{prop:jump-limit} respectively. Also, one can check that the dichotomy exhausts all cases under Assumption~\ref{ass:nondegenerate} (this is clear from the phase diagram on Figure~\ref{fig:phase-diagram}).
\end{proof}

Regarding the case~\ref{case:critical} our ansatz is that the limit $(\rho_t)$ solves
\begin{align}\label{eq:ansatz}
\partial_t \rho_t = \alpha^{-1}\left(\int K[\rho_t](\cdot|x) \d\rho_t(x) -\rho_t\right),\qquad \rho_0=\mu_0
\end{align}
where the transition kernel $K[\rho]$ is defined as
$$
K[\rho](\cdot |x)=\mathrm{Law}(x-\alpha \nabla V[\widehat M(\rho)](x)-\alpha \beta \nabla P(x)).$$ 
Here $\widehat M:\Pp_1(\RR^p)\to\Pp_1(\RR^p)$ is the random map defined as $\widehat M(\rho) = \beta \sum_{i=1}^{N} \delta_{Z_i}$ where $N\sim \text{Poisson}(1/\beta)$ and $Z_1,\dots,Z_{N}$ are independent samples from $\rho$. In the definition of $K[\rho]$, the law is over the randomness of $\widehat M$, which is the only source of randomness in this expression. We formally justify this ansatz in Section~\ref{sec:intuition}, but we leave a rigorous analysis of this case as an open question. Note that this critical limit corresponds to learning rates of order $\Theta(1/n)$ and is therefore of  limited practical relevance. 

The proof of Theorem~\ref{thm:main} is deferred to Section~\ref{sec:proofs} and a discussion of the various limits is the object of Section~\ref{sec:discussions}. For now, we can make the following short remarks:
\begin{enumerate}
\item When proving Theorem~\ref{thm:main}, we in fact obtain in cases~\ref{case:discrete} and~\ref{case:WGF} the stronger notion of convergence at the level of particle's trajectories. In mathematical terms, for any $m\in \NN^*$ fixed, we obtain the convergence in distribution of the family of paths $(X^{1,n},\dots,X^{m,n})$ towards $(Y^1,\dots,Y^m)$ where $(Y^i)$ are independent realizations of a limit process {given in Section \ref{sec:proofs}}. 
\item Interestingly, in all nondegenerate limits,  $\bm f(\rho_t)$ evolves deterministically.
\item Our results are proved for the squared loss, but it would not be difficult to extend them to any smooth loss when $\beta=0$.  When $\beta>0$, the form of the limit would change because the dropout penalty does not have a simple closed form beyond the squared loss.
\end{enumerate}

\section{Qualitative properties of the various limits}\label{sec:discussions}

{In this section, we discuss some properties of the asymptotic processes presented previously.}

\subsection{The penalized Wasserstein gradient flow limit~\ref{case:WGF}{: equivalence with explicit penalization}}
The only two limits that exhibit the dropout penalty, that is when dropout noise remains asymptotically biased (see~\eqref{eq:penalization-def}), are cases~\ref{case:critical} and~\ref{case:WGF} with $\beta>0$. {These cases involve small learning rates $\tau_n=O(1/n)$ and keep rates $q_n=\Theta(1/n)$. Such learning rates are impractical in large-scale settings as they considerably slow down training.}

{If the penalty effect is indeed beneficial and desired, a more practical approach is to dispense with dropout altogether and instead add the deterministic penalty $P(\theta)$ directly to the objective. In the next proposition, we show that this procedure is asymptotically equivalent to dropout in regimes where the penalty emerges, focusing for simplicity on case~\ref{case:WGF}, that is, with $\tau_n=o(1/n)$. However, unlike implementing the penalty via standard dropout, this approach does not require vanishing learning rates and is therefore more widely applicable, as well as being simple to implement.}

\begin{proposition}[Explicit-implicit penalty equivalence]\label{prop:penalty-equivalent}
Let $\beta>0$. Consider an initialization $(\xi^1,\dots,\xi^n)$ as in~\eqref{eq:GD-dropout-weights} and let:
\begin{enumerate}[label=(\roman*)]
\item $(X_k^{i,n})$ the iterates of GD \emph{with} dropout on the loss $\Ll^\eta$ with $q_n=\frac1{\beta n}$ and $\tau_n = o(\frac{1}{n})$; 
\item $(\tilde X_k^{i,n})$ the iterates of GD \emph{without} dropout ($q=1$) on the \emph{penalized} loss  $\theta \mapsto \Ll(\theta)+\frac{\beta}{n}\sum_{i=1}^n P(X^{i,n})$ with the same $(\tau_n)$;
\end{enumerate}
Then we have pathwise convergence over bounded time horizons: for any $T>0$ and $m\in \NN^*$, it holds
$$
\E\Big[\sup_{k\leq T/\tau_n}\sup_{i\leq \min\{m,n\}} \Vert \tilde X_k^{i,n}- X_k^{i,n}\Vert_2\Big] \xrightarrow[n\to \infty]{} 0.
$$
In particular the respective distributional limits $(\rho_{t})$ and $(\tilde \rho_{t})$ exist and are equal $(\rho_{t})=(\tilde \rho_t)$.
\end{proposition}
\begin{proof}
This follows from the fact that the  proof of Proposition~\ref{prop:flow-limit} applies to $X_k^{i,n}$ but also to $\tilde X_k^{i,n}$ (there are just fewer error terms in this case). So both dynamics converge pathwise to the limit process denoted $(Y^1,\dots,Y^m)$ in that statement.
\end{proof}

\subsection{The jump process limits~\ref{case:discrete} \&~\ref{case:jump}}
In this section, we discuss properties of the jump process limits. Since the continuous-time jump process limit~\ref{case:jump} is simply the small learning rate limit of the discrete-time jump process limit~\ref{case:discrete}, we show properties for whichever of these two limits is more convenient; keeping in mind that an analogous property could be studied in the other limit.

\subsubsection{Equivalence with the random metric (RaM) technique}
First, let us show that in these asymptotics, dropout can be equivalently replaced by another, simpler training technique which consists  in randomly removing units only \emph{after} the gradient has been computed. In the two-layer case, this gives the following dynamics  with $i\in [1:n]$ (compare with~\eqref{eq:GD-dropout-weights}),
 \begin{equation}\label{eq:GD-RAM-weights}
 \left\{
 \begin{aligned}
 \tilde X^{i,n}_0 &= \xi^i, \\
 \tilde X^{i,n}_{k} &= \tilde X^{i,n}_{k-1}  -\tau_n\cdot (1+\eta^i_k)D\bm{\phi}(\tilde X^{i,n}_{k-1})^\top \left(\frac1n \sum_{i=1}^n \bm{\phi}(\tilde X^{i,n}_{k-1})-\bm{y}\right),\; k\in \NN^*
 \end{aligned}
 \right.
 \end{equation}
 where $(\eta^i_k)$ is as in~\eqref{eq:GD-dropout-weights}.
 This corresponds to GD but with a random diagonal (degenerate) metric (i.e.~inverse preconditioner) with entries $(1+\eta^i_k)^{-1}$, which takes the two values $\alpha_n^{-1}$ and $+\infty$. {The updates are} equivalent to random block-coordinate descent on the loss, an optimization technique that is usually exploited to leverage the lower computational complexity of sparse updates. Since {the sparsity} is not the motivation here and since the idea can be generalized to metrics with more than two values ($\alpha_n^{-1}$ and $+\infty$), we prefer to refer to this as the \emph{random metric} (RaM) technique.
 
 Let us state the equivalence in the discrete-time case. This result is empirically illustrated on Figure~\ref{fig:dropout-ram}. 
 \begin{proposition}[Dropout-RaM equivalence] 
\label{prop:ram-equivalent}
Let $\tau,q>0$ and $\alpha=\tau/q$. Consider, with the same initialization $(\xi^1,\dots,\xi^n)$ as in~\eqref{eq:GD-dropout-weights}:
\begin{enumerate}[label=(\roman*)]
\item $(X_k^{i,n})$ the iterates of GD with dropout on the loss $\Ll^\eta$ with fixed dropout rate $q$ and learning rate $\tau$;
\item  $(\tilde X_k^{i,n})$ the RaM iterates as in~\eqref{eq:GD-RAM-weights} with $\tau_n=\tau$ and the same dropout masks $(\eta^i_k)$;% and let $(\tilde \rho_t)$ the asymptotic distributional dynamics; 
\end{enumerate}
Then we have pathwise convergence over bounded time horizons: for any $K>0$ and $m\in \NN^*$, it holds
$$
\E\Big[\sup_{k\leq K}\sup_{i\leq \min\{m,n\}} \Vert \tilde X_k^{i,n}- X_k^{i,n}\Vert_2\Big] \xrightarrow[n\to \infty]{} 0.
$$
In particular the respective distributional limits $(\rho_{k})$ and $(\tilde \rho_{k})$ exist and are equal $(\rho_{k})=(\tilde \rho_k)$.
 \end{proposition}
 \begin{proof}
This follows from the fact that the proof of Proposition~\ref{prop:discrete-limit} applies to $X_k^{i,n}$ but also to $\tilde X_k^{i,n}$ (there are just fewer error terms in this case). So both dynamics converge pathwise to the limit process denoted $(Y^1,\dots,Y^m)$ in that statement.
 \end{proof}
 
\begin{figure}
\centering
\includegraphics[scale=0.48]{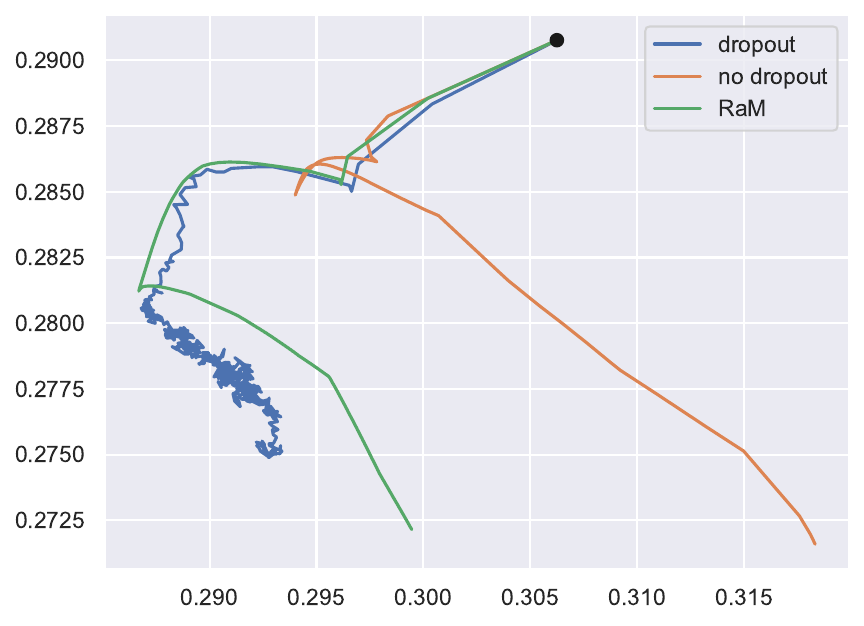}
\hspace{0.5cm}
\includegraphics[scale=0.48]{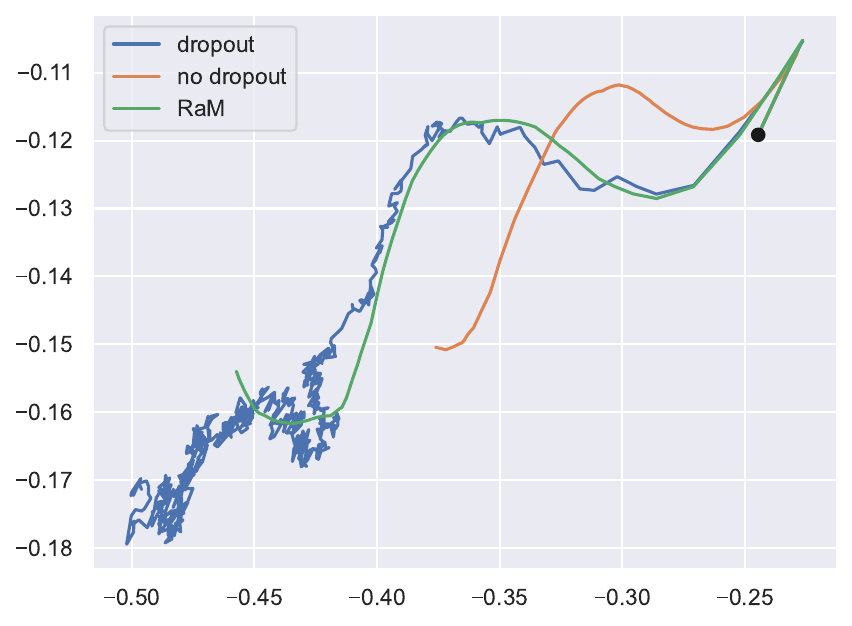}
\caption{Illustration of the pathwise convergence between random metric (RaM) and dropout dynamics (Proposition~\ref{prop:ram-equivalent}). We train a width-5000 two-layer NN on a synthetic teacher-student task with the quadratic loss, either with GD (orange), GD with dropout (blue) or GD with RaM (green), with coupled randomness (of initialization and masks $(\eta^i_k)$). The paths shown correspond to two-dimensional projections of the trajectory of two randomly chosen neurons. The pathwise similarity between dropout and RaM illustrates Proposition~\ref{prop:ram-equivalent} (the similarity degrades at large times because the width is finite). More details and additional plots are given in Appendix \ref{subsec:exp-details-trajs}. }\label{fig:dropout-ram}
\end{figure}
 
\subsubsection{A benefit of RaM: large feature updates beyond the sharpness bound}\label{sec:LLR}
RaM has an interesting consequence in the large learning rate regime: it makes it possible for certain neurons to take larger steps than what the sharpness of the loss would otherwise permit, as long as the average effective learning rate is small enough. To present this idea, let us study the evolution of the loss along one random-metric GD step in the infinite-width limit.
 
For the purpose of this discussion, instead of parameterizing the neural network by a measure $\mu \in \Pp(\RR^p)$, it is more convenient to  parameterize it by a random variable $X\in L^2(\Omega;\RR^p)$, in the spirit of Lions calculus~\citep[Section~6.1]{cardaliaguet2010notes} (here $\Omega$ is an abstract probability space). These two parameterizations lead to the same predictor if $\mathrm{Law}(X)=\mu$. In this parameterization, for $\RR^p$-valued $L^2$ random variables $X,Y_1,Y_2$, we have the following loss and $L^2$-derivatives~\citep[Section~6.1]{cardaliaguet2010notes}
 \begin{align*}
 \Ll(X) &= \frac12 \big\Vert \E [\bm \phi(X)]
-\bm y\big\Vert_2^2,\\
D\Ll(X)(Y_1) &= ( \E [\bm \phi(X)]
-\bm y)^\top \E\big[D\bm \phi(X) [Y_1] \big]\\
D^2\Ll (X)(Y_1,Y_2) &= \E\big[D\bm \phi(X) [Y_1]\big]^\top \E\big[D\bm \phi(X) [Y_2] \big] + ( \E [\bm \phi(X)]
-\bm y)^\top\E\big[D^2\bm \phi(X) [Y_1,Y_2] \big]
\end{align*}
From the expression of the first differential, we obtain that the gradient's expression is $\nabla \Ll(X)=D\bm \phi(X)^\top ( \E [\bm \phi(X)]-\bm y)$ (this is a bounded random variable under Assumption~\ref{ass:phi}). Now given $X$ and an independent nonnegative $L^2$ random variable $\eta$, the random geometry update is of the form $X_+=X-\eta \nabla \Ll(X)$. 
 
 In these notations, the one-step evolution of the loss is given, up to second-order terms, by
 $$
 \Ll(X_+) \approx \Ll(X)- D\Ll(X) (\eta \nabla \Ll(X)) +\frac12  D^2\Ll (X)(\eta \nabla \Ll(X),\eta \nabla \Ll(X)) %+O(\Vert \eta \nabla \Ll(X) \Vert^3).
 $$
 Ignoring higher order terms, the decrease of the objective is guaranteed provided $S(X)\leq 1$ with
 $$
S(X) \coloneqq \frac{D^2\Ll(\eta \nabla \Ll(X),\eta \nabla \Ll(X)) }{2 D\Ll(X) (\eta \nabla \Ll(X))} .%=  \frac{B(X)\E[\eta]^2 + C(X)\E[\eta^2]}{A(X)^2 \E[\eta]^2}
 $$

 Using the independence of $\eta$, the denominator is given by
 $$
 2D\Ll(X) (\eta \nabla \Ll(X))  =  2 \E[\eta] (\E[\bm\phi(X)]-\bm y)^\top \E[D\bm \phi(X)D\bm \phi(X)^\top] (\E[\bm\phi(X)]-\bm y) = A(X)\E[\eta]
 $$
 for some $A(X)\geq 0$.  The numerator is the sum of two terms $(I)+(II)$ with
 \begin{align*}
 (I) &= \frac12 \E\big[D\bm \phi(X) [\eta \nabla \Ll(X)]\big]^\top \E\big[D\bm \phi(X) [\eta \nabla \Ll(X)] \big]\\
 &=\frac12   \E[\eta]^2 \big\Vert\E\big[D\bm \phi(X)D\bm \phi(X)^\top (\E[\bm \phi(X)]-\bm y)\big]\big\Vert^2\\
 & = B(X) \E[\eta]^2
 \intertext{for some $B(X)\geq 0$ and}
 (II) &= \frac12 ( \E [\bm \phi(X)]
-\bm y)^\top\E\big[D^2\bm \phi(X) [\eta\nabla \Ll(X),\eta \nabla \Ll(X)] \big]\\
 &=\frac12   \E[\eta^2] (\E[\bm \phi(X)-\bm y])^\top \E\big[D^2\bm \phi(X)[D\bm \phi(X)^\top(\E[\bm \phi(X)]-\bm y), D\bm \phi(X)^\top(\E[\bm \phi(X)]-\bm y)]\big]\\
 & = C(X) \E[\eta^2]
 \end{align*}
for some $C(X)\in \RR$. It follows that
  $$
S(X)= \frac{B(X)\E[\eta]^2 + C(X)\E[\eta^2]}{A(X) \E[\eta]}
$$
 If $\eta=\eta_0$ is deterministic, then the descent condition $S(X)\leq 1$ is $\eta_0\leq \eta_{\max}\coloneqq \frac{A(X)}{B(X)+C(X)}$ which is a local version of the usual smoothness-based descent lemma in optimization.
 
But when $\eta$ is random, the factor in front of the two terms of the Hessian is different.  Moreover, it is a pervasive empirical observation in the context of neural networks~\cite[e.g.][]{wei2020implicit} that it typically holds $B(X)\gg |C(X)|$. This means that provided one chooses $\E[\eta]$ a bit smaller than $\eta_{\max}$, there is ample room to increase the variance of $\eta$---and therefore allowing randomly chosen neurons to take large steps, beyond the sharpness threshold---while maintaining the descent property. We note that this discussion is valid as soon as $\E[\eta]>0$ and $\E[\eta^2]<+\infty$, in particular the nonnegativity of $\eta$ is not essential.

\subsubsection{Another asymptotic property: dynamical NTK averaging}\label{sec:NTK-time}
Let us mention another interesting property of dropout in the jump process limits: the associated Neural Tangent Kernel (NTK) is smoothed in time via an exponential moving average. Let us present this phenomenon, which is an intuitive consequence of the fact that only a subset of the neurons are updated at each step, for the continuous-time dynamics (case~\ref{case:jump}) where its expression is the cleanest.

As shown in Proposition~\ref{prop:existence-uniqueness-alt}, the limit dynamics $(\rho_t)$ in case~\ref{case:jump} can be characterized as the unique solution to the following integral equation
\begin{align}\label{eq:sequentia-limit-integral-form}
\rho_t = \frac{1}{\alpha}\int_{0}^t e^{(s-t)/\alpha}T_{\alpha}(\rho_s)\d s +e^{-t/\alpha}\mu_0
\end{align}
where $T_{\alpha}(\mu) = (\Id-\alpha\nabla V[\mu])_\# \mu$ is the measure obtained from $\mu$ by taking a single GD step with learning rate $\alpha$  on the loss $\Ll$.% +\beta \int P\d\mu$.

The expression has a counterpart in terms of the NTK. Indeed, the tangent kernel\footnote{With a slight abuse of notation, we identify the NTK with its Gram matrix on the training set.} of a 2NN is a linear function of the measure $\mu \in \Pp(\RR^p)$ that parameterizes it:
$$
\text{NTK}(\mu) = \int D \bm{\phi}(w) D\bm{\phi}(w)^\top \d\mu(w).
$$
After one gradient step with learning rate $\alpha$, the updated tangent kernel is $\text{NTK}(T_{\alpha}(\mu))$. Hence, by~\eqref{eq:sequentia-limit-integral-form}, we see that the tangent kernel at time $t$ writes {as the exponential moving average}
$$
\text{NTK}(\rho_t) = \frac{1}{\alpha} \int_0^t e^{(s-t)/\alpha} \text{NTK}(T_{\alpha}(\rho_s))\d s + e^{-t/\alpha} \text{NTK}(\rho_0).
$$

\subsection{Informal justification of the ansatz for the critical limit~\ref{case:critical}}\label{sec:intuition}
In this section, we justify \emph{informally} the ansatz given in~\eqref{eq:ansatz} for the limit in the critical case $\tau_n\to 0$, $\alpha,\beta>0$. 

At each timestep, there is a probability $q_n$ that a neuron moves and $1-q_n$ that it stays still. When $n$ is large, the distribution of neurons at iteration $k+1$ is therefore of the form
 $
 \hat \mu^{n}_{k+1} = q_n \bar \mu^n_{k+1}+(1-q_n)\hat \mu^{n}_{k}
 $ 
 where $\bar \mu^n_{k+1}$ is the distribution of the new position of the active neurons. Reorganizing this update as $ \frac{1}{\tau_n}(\hat \mu^{n}_{k+1}-\hat \mu^{n}_{k})=\frac{q_n}{\tau_n} (\bar \mu^n_{k+1}-\mu^{n}_{k})$, this suggests a limit equation of the form
 $$
 \partial_t \rho_t = \frac{1}{\alpha} (\bar \rho_t-\rho_t)
 $$
where $\bar \rho_t$ is the distribution of the new positions of the neurons that have moved at time $t$ (this is the same form as the jump process limit~\eqref{eq:large-width-limit-3}). 

To derive the expression for $\bar \rho_t$, let us take a closer look at the update equation~\eqref{eq:GD-dropout-weights} for a single neuron. Conditioning on the neuron $i$ moving at iteration $k$ (that is, conditioning on $(1+\eta^i_{k+1})=1/q_n$), its update can be decomposed as
 \begin{align*}
  X^{i,n}_{k+1} - X^{i,n}_{k} &=   -\tau_n \cdot (1+\eta^i_{k+1})D\bm{\phi}(X^{i,n}_{k})^\top \left(\frac1n \sum_{j=1}^n (1+\eta^j_{k+1}) \bm{\phi}(X^{j,n}_{k})-\bm{y}\right)\\
  &=\underbrace{-\frac{\tau_n}{q_n} D\bm{\phi}(X^{i,n}_{k})^\top \left(\frac1n \sum_{j\neq i} (1+\eta^j_{k+1}) \bm{\phi}(X^{j,n}_{k})-\bm{y}\right)}_{(I)} \underbrace{- \frac{\tau_n}{nq^2_n} D\bm{\phi}(X^{i,n}_{k})^\top\bm{\phi}(X^{i,n}_{k})}_{(II)}
 \end{align*}
 
 The second term $(II)$ is simple: as $n\to \infty$, it converges to $-\alpha\beta \nabla P(X^{i,n}_k)$. The first term $(I)$ is more subtle because the number of terms in the sum follows a $\text{Binom}(n-1,q_n)$ distribution, whose expectation $(n-1)q_n \to 1/\beta$  does not diverge, so there is no limit theorem taking place to simplify this sum. Asymptotically, the sum behaves as drawing $N$ according to a Poisson distribution with parameter $1/\beta$ (this is the limit of $\text{Binom}(n-1,q_n)$ distributions) and drawing $N$ independent samples $Z_1,\dots, Z_N$ according to $\rho_t$  to form the sum $1/(nq_n)\sum_{j=1}^N {\bm \phi}(Z_j)$. Hence the first term as a whole $(I)$ is, in the limit, sampled via the formula $-\alpha \nabla V[\widehat M(\rho_t)](X_k^{i,n})$ where $\widehat M$ is the operator defined below~\eqref{eq:ansatz}.

Finally, the update of all the neurons that move at iteration $k$ are certainly not independent; however, in the limit, since $\tau_n \to 0$ there are infinitely many such moves over an infinitesimal time interval $[t,t+\delta t]$, and they are independent in time (considering that $\rho_t$ changes negligibly over this time interval). This leads to an averaging-over-time effect: the neurons located near $x$ and that move during this time interval are distributed, at time $t+\delta t$, according to
$$
K[\rho_t](\cdot|x) = \text{Law}(x-\alpha \nabla V[\widehat M(\rho_t)](x)-\alpha\beta \nabla P(x) )
$$
which is deterministic. This leads to the expression $\bar \rho_t = \int K[\rho_t](\cdot|x)\d\rho_t (x)$ and concludes the justification of ansatz~\eqref{eq:ansatz}.

\subsection{A synthesis: the three effects of dropout and experimental illustration}\label{sec:numerics}
Our theory allows identifying three distinct effects of dropout, with each limit in the phase diagram characterized by the combination of one or several of these effects. {This can be seen by suitably decomposing the update equation~\eqref{eq:GD-dropout-weights}, which we recall is
$$
 X^{i,n}_{k+1} = X^{i,n}_{k}  -\tau_n\cdot (1+\eta^i_{k+1})D\bm{\phi}(X^{i,n}_{k})^\top \left(\frac1n \sum_{j=1}^n (1+\eta^j_{k+1}) \bm{\phi}(X^{j,n}_{k})-\bm{y}\right).
$$
Introducing an independent copy $\tilde{\eta}^{j}_{k+1}$ of $\eta^{j}_{k+1}$, we get 
\begin{align*}
X^{i,n}_{k+1} &= \underbrace{X^{i,n}_k - \tau_n D\bm \phi(X^{i,n}_k)^\top(\bm f(\hat \mu^n_k)-\bm y)}_{\text{update without dropout}} - \underbrace{\frac{\tau_n(1 + \eta_{k+1}^i)}{n} \sum_{j=1}^n \tilde{\eta}^{j}_{k+1} D\bm\phi(X^{i,n}_k)^\top \bm\phi(X^{j,n}_k)}_{\text{PN}} \\
 &-\underbrace{\tau_n \eta^{i}_{k+1} D\bm \phi(X^{i,n}_k)^\top (\bm f(\hat \mu^{n}_k)-\bm y)}_{\text{RaM}} + \underbrace{\tau_n \frac{\eta^i_{k+1} \tilde{\eta}^{i}_{k+1} -(\eta^i_{k+1})^2}{n} D\bm \phi(X^{i,n}_k)^\top \bm \phi(X^{i,n}_k)}_{\text{Penalty}}.
\end{align*}
Let us now give an interpretation of each of the terms:
}
\begin{enumerate}
\item \textbf{Propagation noise {(PN) - active in case~\ref{case:critical}}.}  This is the perturbation of the signal in the forward pass due to the random thinning and rescaling {of the forward pass}. It manifests itself in the randomness of the predictor. In our context, it is a centered perturbation of the update of units of variance ${\Theta(\frac{\tau^2}{n^2}\E[(1+\eta)^2]\E[(\sum_{i} \eta^i)^2])}=\Theta(\tau \alpha\beta)$. 
\item \textbf{Random metric (RaM) { - active in cases~\ref{case:discrete},~\ref{case:jump}, and~\ref{case:critical}}.} This is, given a full gradient, the effect of its random thinning and rescaling. It is due to the mask applied in the backward pass. It leads to a centered perturbation of the update of all units of variance $\Theta(\tau^2\E[\eta^2])=\Theta(\tau\alpha)$.
\item \textbf{Dropout penalty {- active in cases~\ref{case:WGF} and~\ref{case:critical}}.} This is the bias of the dropout gradient, which is precisely the effect of the non-independence between the masks used in the forward and the backward pass. {In the limit, the randomness of this term vanishes, and} it leads to a deterministic perturbation of the update of all units of average magnitude $\Theta(\tau \beta)$ {(see proof of Proposition~\ref{prop:flow-limit} for details)}.
\end{enumerate}
{This decomposition} can be used to intuit the phase diagram of Figure~\ref{fig:phase-diagram}, by separating the cases according to whether $\alpha$ only is positive ({case~\ref{case:discrete} with nonvanishing learning rate and  case~\ref{case:jump} with vanishing learning rate}) or $\beta$ only is positive (case~\ref{case:WGF} with $\beta>0$) or both (case~\ref{case:critical}) or neither (case~\ref{case:WGF} with $\beta=0$).

{To give experimental insight on the effect of each of these terms,} we plot in Figure~\ref{fig:risk} the training logs for a {width 4000} two-layer ReLU NN trained on binary classification on two classes of MNIST with the logistic loss (see details in Appendix \ref{subsec:exp-details-mnist}). The  training configurations illustrate various combinations of the three effects presented above:  
\begin{itemize}
\item  {``PN''} refers to dropout applied in the forward pass only, and shows the isolated effect of propagation noise.
\item  {``RaM''} refers to dropout applied in the backward pass only, and shows the isolated effect of RaM.
\item  {``PN + RaM''} refers to dropout applied in both the forward and the backward passes, but with independent masks. This shows the combined effect of propagation noise and RaM. The difference  {with} ``dropout'' is precisely the effect of the dropout penalty.
\item  {``dropout'' corresponds to standard dropout, which combines the effect of PN, RaM, and dropout penalty, while ``no dropout'' refers to vanilla SGD.}
\end{itemize}
We observe that RaM almost achieves the performance of dropout. Adding dropout in the forward pass does not improve, or even degrades, performance, unless one uses the same masks in the forward and backward passes---that is standard dropout. In the early phase of training, the dropout, RaM and PN+RaM curves match, which suggests that the speed-up observed for dropout in this early phase is entirely due to RaM.

\begin{figure}
\centering
\begin{subfigure}{0.5\linewidth}{}
\centering
\includegraphics[scale=0.6]{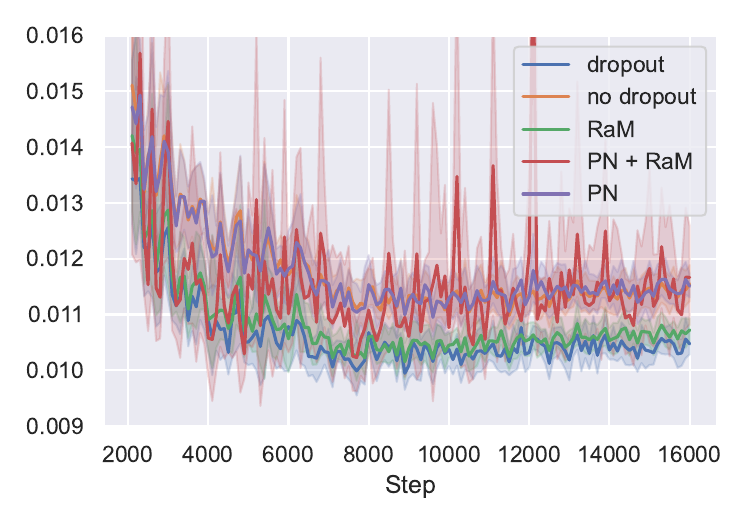}
\caption{Test loss history ($q=0.6$)}
\end{subfigure}%
\begin{subfigure}{0.5\linewidth}{}
\centering
\includegraphics[scale=0.6]{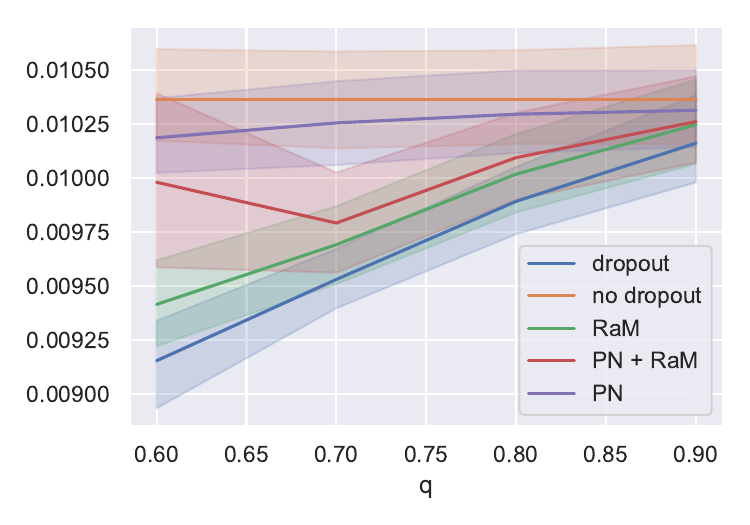}
\caption{Best test loss}
\label{fig:best-test-loss-mnist}
\end{subfigure}%
\caption{Comparison of variants of dropout for the training of two-layer NNs on two classes of MNIST with logistic loss {for several variants of dropout (see details in text), as a function of training steps (left) and of keep probability (right)}. More details are given in Appendix~\ref{subsec:exp-details-mnist}.}\label{fig:risk}
\end{figure}

\section{Proof of Theorem~\ref{thm:main}}\label{sec:proofs}

 Let us begin with some preliminary results which will be useful in all limits.

\subsection{Preliminaries}\label{sec:preliminaries}
The following consequence of Assumption~\ref{ass:phi} will be used repeatedly in the proofs, often without explicit reference. 
\begin{lemma}\label{lem:MF-Lip}
Under Assumption~\ref{ass:phi}, there exists $C>0$ such that, for any $\mu,\mu'\in \Pp_1(\RR^p)$ and $x,x'\in \RR^p$ it holds
\begin{align}\label{eq:field-regularity}
\Vert \nabla V[\mu](x) -\nabla V[\mu'](x')\Vert_2 &\leq C(\Vert x-x'\Vert_2+W_1(\mu,\mu')),
\end{align}
and $\Vert \nabla V[\mu](x)\Vert_2<C$, $\Vert \nabla P(x)\Vert_2\leq C$, $\Ll(\mu)<C$.
\end{lemma}

\begin{proof}
We recall that Assumption~\ref{ass:phi} states that $\bm \phi$ and $D\bm \phi$ are uniformly bounded. The last three bounds directly follow from the expressions $\bm f(\mu)=\int \bm \phi\d\mu$, $\nabla V[\mu](x)=D\bm \phi(x)^\top (\bm f(\mu)-\bm y)$, $\nabla P(x)=D\bm \phi(x)^\top\bm \phi(x)$ and $\Ll(\mu) = \frac12 \Vert \bm f(\mu)-\bm y\Vert_2^2$. To prove~\eqref{eq:field-regularity}, we proceed as follows:
\begin{align*}
\Vert \nabla V[\mu](x) - \nabla V[\mu'](x')\Vert_2 &= \Vert D \bm \phi(x)^\top (\bm f(\mu)-\bm y) - D \bm \phi(x')^\top (\bm f(\mu')-\bm y) \Vert_2 \\
&\leq \Vert D \bm \phi(x)^\top (\bm f(\mu)-\bm y) - D \bm \phi(x')^\top (\bm f(\mu)-\bm y) \Vert_2\\
&\qquad +\Vert D \bm \phi(x')^\top (\bm f(\mu)-\bm y) - D \bm \phi(x')^\top (\bm f(\mu')-\bm y) \Vert_2\\
&\leq \Vert D \bm \phi(x)-D \bm \phi(x')\Vert_{2\to 2}\cdot \Vert \bm f(\mu)-\bm y\Vert_2 \\
&\qquad+ \Vert D\bm \phi(x')\Vert_{2\to 2} \Vert \bm f(\mu)-\bm f(\mu')\Vert_2\\
&\leq \Lip(D\bm \phi)\cdot \sqrt{2\Ll(\mu)}\cdot \Vert x-x'\Vert_2+\Lip(\bm \phi)\Vert \bm f(\mu)-\bm f(\mu')\Vert_2\\
&\leq \Lip(D\bm \phi)\cdot (\Vert \bm \phi\Vert_\infty+\Vert \bm y\Vert_2)\cdot \Vert x-x'\Vert_2+\Lip(\bm \phi)^2W_1(\mu,\mu').\quad \qedhere
\end{align*}
\end{proof}

We note that the term $W_1(\mu,\mu')$ gives in general pessimistic estimates when it comes to quantitative convergence estimates to the mean-field limit, but this is not our goal in this paper so we adopt this bound for convenience. See~\cite{mei2019mean} or~\cite{chizat2025hidden} for the techniques to obtain tight error bounds.

We also recall a $W_1$ version of the law of large numbers for empirical distributions.
\begin{lemma}\label{lem:varadarajan}
Let $(\Xx,\dist)$ be a Polish space (that is, a separable metric space that admits a complete metric equivalent to $\dist$), let $\mu \in \Pp_1(\Xx)$, let $X^1,X^2,\dots$ be independent samples from $\mu$, and let $\hat \mu^n = \frac{1}{n}\sum_{i=1}^n \delta_{X^i}$ be the corresponding sequence of empirical measures. Then it holds $W_1(\mu,\hat \mu^n)\xrightarrow[]{a.s.} 0$ (hence, in particular $\E[W_1(\mu,\hat \mu^n)]\to 0$).
\end{lemma}
\begin{proof}
We provide a proof of this well-known fact for completeness. According to the law of large numbers for empirical measures~\citep{varadarajan1958convergence} $\hat \mu^n$ converges weakly (in duality with continuous and bounded functions) to $\mu$, almost surely. Moreover, posing $x_0\in \Xx$ and $\psi(x)= {\dist}(x_0,x)$, we have $\int \psi(x)\d\hat \mu^n = \frac1n \sum_{i=1}^n \psi(X^i)$ which is a sum of i.i.d.~random variables with finite first moment because $\mu\in \Pp_1(\Xx)$. So, by the strong law of large numbers, $\int  {\dist}(x_0,x)\d\hat \mu^n(x)$ converges to $\int {\dist}(x_0,x)\d \mu(x)$ almost surely. It follows, by~\citet[Theorem~6.9]{villani2008optimal} that $W_1(\hat \mu^n,\mu)\xrightarrow[n\to \infty]{}0$ almost surely.
\end{proof}

We will consider the following Polish spaces:
\begin{itemize}
\item The usual Euclidean space $((\RR^p)^K,\Vert \cdot \Vert_2)$;
\item The space of continuous paths $(\Cc([0,t];\RR^p),d_t)$ where $d_t(x,x')=\sup_{s\leq t} \Vert x(s)-x'(s)\Vert_2$;
\item The Skorokhod space $(\Dd([0,t],\RR^p), d'_t)$ of càdlàg functions (right-continuous functions with left limits at all points) where $d'_t$ is the Skorokhod metric~\citep[Chap.~3]{billingsley2013convergence}: 
$$
d_t'(x,x') = \inf_{\lambda\in \Lambda} \sup_{s\in [0,t]} \vert \lambda(s)-s\vert + \sup_{s\in [0,t]} \Vert x(\lambda(s))-x'(s)\Vert_2
$$
where $\Lambda$ is the set of strictly increasing continuous maps from $[0,t]$ to $[0,t]$.
\end{itemize}
The topologies above are those that we associate to these spaces throughout. %\LC{check if we use it for skorokhod}

\subsection{Case~\ref{case:discrete} : the discrete-time jump process limit}
In this section we assume that $n\to \infty$ and that $ q_n \to q\in ]0,1]$ and $\tau_n \to \tau>0$. This is the simplest limit to study, in particular because it remains discrete-in-time in the limit.

\paragraph{Limit dynamics.} Consider the (discrete-time) stochastic process $(Y_k)_{k\in \NN}$ defined recursively by $Y_0\sim \mu_0$ and 
\begin{align}   \label{eq:def-discrete-dynamics}
Y_{k+1} = \begin{cases}
Y_k &\text{with probability $1-q$}\\
Y_k - \alpha \nabla V[\rho_k](Y_k)&\text{with probability $q$}\\
\end{cases},&& \rho_k = \mathrm{Law}(Y_k).
\end{align}
By construction, it holds 
\begin{align*}
\rho_0=\mu_0, &&\rho_{k+1} =q(\Id-\alpha \nabla V[\rho_k])_\#\rho_k + (1-q)\rho_k, \; \forall k\geq 1.
\end{align*}
Then the discrete-time jump process limit in the main theorem (Theorem~\ref{thm:main}\ref{case:discrete}) is a consequence of the following, more precise statement:
\begin{proposition}\label{prop:discrete-limit}
Let Assumption~\ref{ass:phi} hold and assume that $ q_n \to q\in ]0,1]$, that $\tau_n \to \tau>0$. For any $k\in \NN$,  it holds $\E[W_1(\hat \mu^n_k,\rho_k)]\xrightarrow[n\to \infty]{}0$. Moreover, for any fixed $m\in \NN^*$ and time horizon $K\in \NN^*$, it holds
$$
(X^{1,n},\dots,X^{m,n}) \xrightarrow[n\to \infty]{\text{distr.}} (Y^1,\dots,Y^m)
$$
where $(Y^1,\dots,Y^m)$ are independent copies of the $(\RR^p)^{K+1}$-valued random variable $Y$.
\end{proposition}
\begin{proof}
Let us consider constant sequences $\tau_n=\tau$ and $q_n=q$, as the general case can be proved along the same lines.

\paragraph{Preparation.} The update equation~\eqref{eq:GD-dropout-weights} can be decomposed, for $k\in \NN$, as
 \begin{align*}
  X^{i,n}_{k+1} &= X^{i,n}_{k}  -\tau \cdot (1+\eta^i_{k+1})D\bm{\phi}(X^{i,n}_{k})^\top \left(\frac1n \sum_{j=1}^n (1+\eta^j_{k+1}) \bm{\phi}(X^{j,n}_{k})-\bm{y}\right)\\
  &=X^{i,n}_{k}  -\tau \cdot (1+\eta^i_{k+1}) \left(\nabla V[\hat \mu^n_k](X^{i,n}_{k}) +\gamma^{i,n}_{k+1}\right)
 \end{align*}
 where we have gathered errors terms in $\gamma^{i,n}_k$ defined as
\begin{align*}
\gamma^{i,n}_k 
&= {\frac1n \sum_{j=1}^n \eta^{j}_k D\bm \phi(X^{i,n}_{k-1})^\top \bm \phi(X^{j,n}_{k-1})}.
\end{align*}
We have under Assumption~\ref{ass:phi}
\begin{align}\label{eq:proof-i-gamma}
\E \Vert\gamma^{i,n}_k\Vert_2\leq \sqrt{\E \Vert\gamma^{i,n}_k\Vert_2^2} =\sqrt{O\Big({\frac{n\E [\eta^2]}{n^2}}\Big)} = O(\sqrt{\beta_n})\xrightarrow[n\to \infty]{} 0.
\end{align}

\paragraph{Coupling argument.}
Consider $n$ independent samples from $Y$, labelled $(Y^1,\dots,Y^n)$ which we couple with  $(X^{1,n},\dots, X^{n,n})$ via the initialization and the random dropout masks as follows: $Y^i_0=X^{i,n}_0=\xi^i $ and
$$
Y^{i}_{k+1} =Y^{i}_{k} -\tau (1+\eta^i_{k+1}) \nabla V[\rho_k](Y^{i}_{k}),
$$
{where we recall that $\rho_k$ is defined in \eqref{eq:def-discrete-dynamics}.}
Observe that the definition of $Y^{i}_{k+1}$ involves the same dropout masks as that of $X^{i,n}_{k+1}$. Let us also define $\hat \rho^n_\ell=\frac1n \sum_{i=1}^n \delta_{Y^i_\ell}$.

By construction, it holds $\Vert X^{i,n}_{0}-Y^i_{0}\Vert_2=0$ and for $k\geq 1$, using Lemma~\ref{lem:MF-Lip} and the fact that $|1+\eta|\leq \frac1q$,
\begin{align}
\Vert X^{i,n}_{k+1}-Y^i_{k+1}\Vert_2 &\leq \Vert X^{i,n}_{k}-Y^i_{k}\Vert_2+ \frac{\tau}{q} \Vert \nabla V[\hat \mu^n_{k}](X^{i,n}_{k}) -  \nabla V[\rho_{k}](Y^{i}_{k})\Vert_2+  \frac{\tau}{q}\Vert \gamma_{k+1}^{i,n}\Vert_2 \nonumber \\
&\leq (1+C\alpha)\Vert X^{i,n}_{k}-Y^i_{k}\Vert_2 + C\alpha W_1(\hat \mu^n_{k}, \rho_{k})+ \alpha \Vert \gamma_{k+1}^{i,n}\Vert_2\label{eq:proof-(i)-I}\\
&\leq \sum_{\ell=0}^{k} (1+\alpha C)^{k-\ell} (\alpha C W_1(\hat \mu^n_{\ell}, \rho_{\ell})+\alpha \Vert \gamma_{\ell+1}^{i,n}\Vert_2)\label{eq:proof-(i)-II}.
\end{align}
Denoting $h^n_k = \frac1n \sum_{i=1}^n \Vert X^{i,n}_{k}-Y^i_{k}\Vert_2$, it holds $h^n_0=0$, and we deduce from~\eqref{eq:proof-(i)-I} and from the inequality $W_1(\hat \mu^n_k,\hat \rho^n_k)\leq h^n_k$ that 
\begin{align*}
h^n_{k+1} &\leq (1+C\alpha) h_k^n + C\alpha W_1(\hat \mu^n_k, \hat \rho^n_k)+ C\alpha W_1(\hat \rho^n_k, \rho_k)+ \frac{\alpha}{n}\sum_{i=1}^n \Vert \gamma_{k+1}^{i,n}\Vert_2\\
&\leq (1+2C\alpha) h_k^n + E^n_k
\end{align*}
where $E^n_k:=C\alpha W_1(\hat \rho^n_k, \rho_k)+ \frac{\alpha}{n}\sum_{i=1}^n \Vert \gamma_{k+1}^{i,n}\Vert_2$. Applying this inequality recursively (a discrete-time Gr\"onwall argument), we get
$$
h_k^n \leq\sum_{\ell=0}^{k{-1}} (1+2C\alpha)^{k{-1}-\ell} E^n_\ell \leq (1+2C\alpha)^{k} \sum_{\ell=0}^{k{-1}} E^n_\ell.
$$
Since $\E [W_1(\hat \rho_\ell^n,\rho_\ell)]\to 0, \forall \ell\in \NN$, (Lemma~\ref{lem:varadarajan}), and $\E [\Vert \gamma_{k+1}^{i,n}\Vert_2] \to 0$ (by~\eqref{eq:proof-i-gamma}), we have $\E [E^n_\ell]\to 0$. We deduce that 
\begin{align*}
\E[h^n_k]\leq (1+2C\alpha)^{k} \sum_{\ell=0}^{k{-1}} \E[E_\ell^n]\xrightarrow[n\to \infty]{} 0, \; \forall k\in \NN.
\end{align*}
 It follows
$$
\E[W_1(\hat \mu^n_k,\rho_k)]\leq \E[W_1(\hat \mu^n_k,\hat \rho_k^n)]+\E[W_1(\hat \rho_k^n,\rho_k)] \leq \E[h^n_k]+\E[W_1(\hat \rho_k^n,\rho_k)] \xrightarrow[]{}0
$$
which is the first claim.  For the second claim, we have by~\eqref{eq:proof-(i)-II} that for $m\in \NN^*$,
\begin{align*}
\E\Big[\max_{k\in [1:K]}\max_{i\in [1: m]} \Vert X^{i,n}_k-Y^i_k\Vert_2\Big] \leq \sum_{\ell=0}^{{K-1}} (1+\alpha C)^{{K-1}-\ell} \E\Big[\alpha C W_1(\hat \mu^n_{\ell}, \rho_{\ell}) &+ \alpha  \max_{i\in [1: m]} \Vert \gamma_{\ell+1}^{i,n}\Vert_2\Big] \\
&\xrightarrow[n\to \infty]{} 0.    
\end{align*}
It follows that $(X^{1,n},\dots,X^{m,n})$ converges in $L^1$ to $(Y^1,\dots,Y^m)$. This implies the second claim.
\end{proof}

\subsection{Case~\ref{case:WGF} : the Wasserstein gradient flow limit}
In this section we assume that $n\to \infty$, that $\tau_n\to 0$ and that $ \alpha_n=\frac{\tau_n}{q_n} \to 0$. Because $\alpha_n$, which represents the average time period between two updates of a neuron -- vanishes in the limit, we get a flow-like dynamics in the limit, with a continuous trajectory for each neuron.

\subsubsection{Precise statement and proof}

\paragraph{Limit dynamics.} Consider the continuous-time stochastic process $Y$ that satisfies $Y_0\sim \mu_0$ and 
\begin{align*}%\label{eq:limit-process-ii}
\d Y_t = -\nabla V[\rho_t](Y_t)\d t -\beta \nabla P(Y_t) \d t,\qquad \rho_t = \mathrm{Law}(Y_t)
\end{align*}
where we recall that $P(x)=\frac12 \Vert \bm \phi(x)\Vert_2^2$ (hence $\nabla P(x) = D\bm \phi(x)^\top \bm \phi(x)$).
This is a MacKean-Vlasov process of a simple type because it is noiseless and the only source of randomness is the initialization. Under Assumption~\ref{ass:phi}, the existence and uniqueness of strong solutions is guaranteed (see for instance~\citep{bauer2019existence} which deals with much weaker spatial regularity). Also we have for any time horizon $T$, $\mathrm{Law}(Y) \in \Pp_1(\Cc([0,T];\RR^p))$. It is moreover not difficult to show that $\rho = (\rho_t)_t$ belongs to $\Cc(\RR_+;\Pp_1(\RR^p))$ and that it is the unique solution in the weak sense\footnote{By ``weak solution'', we mean that $t\mapsto \rho_t$ is absolutely continuous for $W_1$ and that for any $\psi\in \Cc^1(\RR^p)$ and a.e.~$t\in \RR_+^*$, it holds $\frac{\d}{\d t} \int \psi \d\rho_t = -\int \nabla \psi \cdot \nabla [V[\rho_t] +\beta P]\d\rho_t$, see e.g.~\cite[Sec.~4.1.2]{santambrogio2015optimal}.} of
\begin{align*}%\label{eq:limit-flow-ii}
\partial_t \rho_t = \nabla \cdot (\rho_t \nabla [V[\rho_t] +\beta P]),\qquad \rho_0=\mu_0.
\end{align*}
If $\mu_0\in \Pp_2(\RR^p)$ (that is, if $\mu_0$ has finite second moments), then this dynamics can be interpreted as the \emph{Wasserstein ($W_2$) gradient flow}~\cite[Chap.~8]{santambrogio2015optimal} of the penalized risk
$$
\mu\mapsto \frac12 \Vert \bm f(\mu)-\bm y\Vert_2^2 +\beta \int P(x)\d\mu(x).
$$ 

\paragraph{Interpolated dynamics.} Our statement will also involve the following continuous-time process. We denote by $(\bar X^{i,n})$ the piecewise affine interpolation of the discrete-time dynamics, defined by
$$
\bar X_t^{i,n} = X_k^{i,n} + \frac{t-k\tau_n}{\tau_n} (X_{k+1}^{i,n}-X_k^{i,n}), \quad \text{for $t\in [k\tau_n, (k+1)\tau_n[$}.
$$
The gradient flow limit of the main theorem (Theorem~\ref{thm:main}\ref{case:WGF}) is then a consequence of the following, more precise statement:
\begin{proposition}\label{prop:flow-limit}
Assume that $\tau_n\to 0$, that $ \alpha_n =\frac{\tau_n}{q_n}\to 0$, and Assumption~\ref{ass:phi}. Then for any $t\geq 0$, it holds $\E[W_1(\hat \mu^n_{\lfloor t/\tau{_n}\rfloor},\rho_t)]\xrightarrow[n\to \infty]{}0$. Moreover, for any fixed $m\in \NN^*$ and time horizon $T>0$, it holds (in $\Pp(\Cc([0,T];\RR^p)^m)$)
$$
(\bar X^{1,n},\dots, \bar X^{m,n}) \xrightarrow[n\to \infty]{\text{distr.}} (Y^1,\dots,Y^m)
$$
where $(Y^1,\dots,Y^m)$ are independent copies of $Y$.
\end{proposition}

\begin{proof}
In order to better highlight the structure of the argument, we obtain this result as a corollary of a general and stronger \emph{mean-field stochastic approximation} result, Proposition~\ref{prop:MF-SA}, stated and proved the next subsection. Let us here show that the assumptions of that proposition are satisfied, that is, that we can indeed decompose the update as in~\eqref{eq:discrete-ii}, with $\gamma^{i,n}_k$ a Martingale difference sequence (MDS) satisfying $\max_{k\leq T/\tau_n, i\leq n} \EE \Vert \gamma^{i,n}_k\Vert_2^2=o(\tau_n)$. 

By suitably decomposing the update of~\eqref{eq:GD-dropout-weights}, which we recall is
$$
 X^{i,n}_{k+1} = X^{i,n}_{k}  -\tau_n\cdot (1+\eta^i_{k+1})D\bm{\phi}(X^{i,n}_{k})^\top \left(\frac1n \sum_{j=1}^n (1+\eta^j_{k+1}) \bm{\phi}(X^{j,n}_{k})-\bm{y}\right),
$$
we get $X^{i,n}_{k+1}=X^{i,n}_k+\tau_n b_n(\hat \mu^n_k,X^{i,n}_k)+\gamma^{i,n}_{k+1}$ with $b_n$ defined by
\begin{align*}
b_n(\mu,x) & =- D\bm \phi(x)^\top(\bm f(\mu)-\bm y) -\frac{\E[(\eta^i_{k+1})^2]}{n} D\bm \phi(x)^\top\bm \phi(x)\\
&= - D\bm \phi(x)^\top(\bm f(\mu)-\bm y) -\frac{(1-q_n)}{nq_n} D\bm \phi(x)^\top \bm\phi(x)
\end{align*}
and
\begin{align*}
\gamma^{i,n}_{k+1} &= -\underbrace{\tau_n \eta^{i}_{k+1} D\bm \phi(X^{i,n}_k)^\top (\bm f(\hat \mu^{n}_k)-\bm y)}_{(I)}  - \underbrace{\tau_n \frac{(\eta^i_{k+1})^2-\E(\eta^i_{k+1})^2}{n} D\bm \phi(X^{i,n}_k)^\top \bm \phi(X^{i,n}_k)}_{(II)}\\
 & \underbrace{ -\frac{\tau_n}{n} \sum_{\substack{j=1}}^n \eta^{j}_{k+1} D\bm\phi(X^{i,n}_k)^\top \bm\phi(X^{j,n}_k)}_{(III)}  \underbrace{- \frac{\tau_n(\eta_{k+1}^i)}{n} \sum_{\substack{j=1\\ j\neq i}}^n \eta^{j}_{k+1} D\bm\phi(X^{i,n}_k)^\top \bm\phi(X^{j,n}_k)}_{(IV)}.
\end{align*}
\begin{itemize}
\item Under Assumption~\ref{ass:phi}, it is clear that $(b_n)$ converges uniformly to $b: (\mu,x)\mapsto -D\bm \phi(x)^\top (\bm f(\mu)-\bm y)-\beta D\bm \phi(x)^\top \bm \phi(x) = -\nabla V[\mu](x)-\beta \nabla P(x)$ and that the latter satisfies the required regularity assumptions by Lemma~\ref{lem:MF-Lip}.
 \item By construction, $(\gamma^{\cdot,n}_k)_k$ is a $(\RR^p)^n$-valued MDS. It remains to check that it satisfies $\max_{k\leq T/\tau_n, i\leq n} \EE \Vert \gamma^{i,n}_k\Vert_2^2=o(\tau_n)$. Since by Cauchy-Schwartz's inequality, a family of centered random variables $(A_i)_{i=1}^k$ satisfies $\E[(\sum_{i=1}^k A_i)^2]\leq \E[k \sum_{i=1}^k A_i^2] = k\sum_{i=1}^k \E[A_i^2]$, it is sufficient to check this property for the terms (I) to (IV) separately. Under Assumption~\ref{ass:phi} it holds
 \begin{enumerate}
 \item $\E \Vert (I)\Vert^2_2 = O(\E[(\tau_n\eta)^2])=O(\tau_n^2\frac{1-q_n}{q_n})=O(\alpha_n\tau_n)=o(\tau_n)$ since $\alpha_n\to 0$;
 \item $\E \Vert (II)\Vert_2^2 \leq O(\frac{\tau_n^2}{n^2}\E[(\eta^2-\E\eta^2)^2])$. Note that $\E\eta^4=O(1/q_n^3)$ and $\E \eta^2 =O(1/q_n)$ so $\E[(\eta^2-\E\eta^2)^2] = \E[\eta^4]-(\E[\eta^2])^2 =O(1/q_n^3)$. Overall $\E \Vert (II)\Vert_2^2=O(\frac{\tau^2_n}{n^2q^3_n}) = O(\tau_n\alpha_n \beta^2_n)=o(\tau_n)$, again because $\alpha_n\to 0$;
 \item $\E \Vert (III)\Vert_2^2 = O(\frac{\tau^2_n}{n^2}\E(\sum_{j} \eta^j)^2)=O(\frac{\tau^2_n}{n^2}\sum_{j}\E (\eta_j)^2)=O(\frac{\tau_n^2}{n^2}\frac{n}{q_n})= O(\frac{\tau_n \alpha_n}{n})=o(\tau_n)$;
 \item $\E \Vert (IV)\Vert_2^2 = O(\frac{\tau_n^2}{n^2}\E[(\eta_i\sum_{j\neq i} \eta^j)^2]) = O(\frac{\tau_n^2}{n^2}\E[(\eta_i)^2] \E[(\sum_{j\neq i} \eta^j)^2])= O(\frac{\tau_n^2}{n^2}(n-1)(\E \eta^2)^2=O(\frac{\tau_n^2}{nq_n^2})=O(\tau_n \alpha_n \beta_n)=o(\tau_n)$.
 \end{enumerate}
  Since all these terms are $o(\tau_n)$, this proves that  Proposition~\ref{prop:MF-SA} applies, and thus concludes the proof. \qedhere
 \end{itemize}
\end{proof}

\subsubsection{A mean-field stochastic approximation result}\label{eq:mean-field-sto-app}
In this section, we prove the general \emph{mean-field stochastic approximation} result used in the previous section. Consider a sequence of step-sizes $(\tau_n)$ going to $0$ and a measure-dependent drift vector field $b:\Pp(\RR^p)\times \RR^p \to \RR^p$ that is uniformly bounded and such that there exists $C>0$ such that
\begin{align*}
\Vert b(\mu,x)-b(\mu',x')\Vert_2 \leq C(W_1(\mu,\mu') +\Vert x-x'\Vert_2), &&\forall \mu,\mu\in \Pp_1(\RR^p),\; x,x'\in \RR^p.
\end{align*}
Consider also a sequence of velocity fields $b_n$ that converges uniformly on $\Pp_1(\RR^p)\times \RR^p$ to $b$.

Let $\mu_0\in \Pp(\RR^p)$ and let $\xi^1,\xi^2,\dots$ a sequence of i.i.d.~samples from $\mu_0$. Consider $\RR^p$-valued martingale difference sequences (MDS) $(\gamma^{i,n}_{k})_{k\in \NN^*}$ for $n\in \NN^*$ and $i\in [1:n]$. This means that we assume for all $n \in \NN^*$, $i\in [1:n]$, and $k\in \NN^*$:
\begin{enumerate}
\item  $\E[\Vert \gamma^{i,n}_k\Vert_2]<+\infty$, and
\item $\E[\gamma^{i,n}_{k}|\Ff^n_{k-1}]=0$
\end{enumerate} 
where the filtration is $\Ff^n_k = \sigma (\{ \xi^i \}_{i=1}^n, \{\gamma^{i,n}_1\}_{i=1}^n,\dots, \{\gamma^{i,n}_k\}_{i=1}^n)$ and $\Ff^n_0 = \sigma (\{ \xi^i \}_{i=1}^n)$. 

\paragraph{Discrete and limit dynamics.} We consider the sequence of $\RR^p$-valued stochastic processes $(X^{1,n},\dots, X^{n,n})$ indexed by $n$ and defined for all $n \in \NN^*$, $i\in [1:n]$ and $k\in \NN^*$ by
\begin{align}\label{eq:discrete-ii}
X^{i,n}_{k+1} = X^{i,n}_k + \tau_n b_n(\hat \mu_k^n,X^{i,n}_k) +  \gamma^{i,n}_{k+1},&& X^{i,n}_0 = \xi^i,&& \hat \mu_k^n = \frac1n \sum_{i=1}^n \delta_{X^{i,n}_k}.
\end{align}
Our goal is to prove, under suitable assumptions, that this process converges in some sense to the continuous-time $\RR^p$-valued stochastic process solving
\begin{align*}
dY_t = b(\rho_t,Y_t)\d t, && Y_0\sim \mu_0, && \rho_t = \mathrm{Law}(Y_t).
\end{align*}
The classical theory of MacKean-Vlasov equations guarantees that there is a unique strong solution $Y$ for this system with $\rho=(\rho_t)_t\in \Cc(\RR_+;\Pp_1(\RR^p))$. Moreover, for any time horizon $T>0$, we have $\mathrm{Law}(Y) \in \Pp_1(\Cc([0,T];\RR^p))$ (e.g.~\citep{bauer2019existence}). 

\paragraph{Interpolated dynamics.} In order to formalize the sense in which $(X^{i,n})$ converges to $Y$, let us map the discrete-time processes $(X^{i,n})$ into continuous-time ones via a piecewise affine interpolation of the trajectories. For $k\in \NN$, let $t_k = \tau_n k$, and for $t\in [t_k,t_{k+1}[$ let
\begin{align*}
\bar X^{i,n}_t &= X^{i,n}_{k} +\frac{t-t_k}{\tau_n} (X^{i,n}_{k+1} -X^{i,n}_{k} )\\
&=X^{i,n}_{k} +(t-t_k) b_n(\hat \mu_k,X^{i,n}_k) + \frac{(t -t_k)}{\tau_n} \gamma^{i,n}_{k+1}
\end{align*}
Finally, we define $\bar \mu^n_t = \frac1n \sum_{i=1}^n \delta_{\bar X^{i,n}_t}$ which interpolates the sequence of discrete measures $(\hat \mu^n_k)_k$. The next result guarantees that the discrete-time finite-particle dynamics converges to the continuous-time mean-field dynamics, provided the MDS terms are small enough.%and $\hat \nu_t^n = \frac1n \sum_{i=1}^n \delta_{Y^i_t}$.

\begin{proposition}\label{prop:MF-SA}
Fix a time horizon $T>0$. As $n\to \infty$, assume that $\max_{k\leq T/\tau_n, i\leq n} \EE \Vert \gamma^{i,n}_k\Vert_2^2=o(\tau_n)$. Then $\E [\sup_{t\leq T} W_1(\bar \mu^n_t,\mu_t)] \xrightarrow{}{} 0$. Moreover, for any fixed $m\in \NN^*$, it holds
$$
(\bar X^{1,n},\dots, \bar X^{m,n}) \xrightarrow[n\to \infty]{\text{distr.}} (Y^1,\dots,Y^m),
$$
where $(Y^1,\dots,Y^m)$ is a family of $m$ independent copies of  $Y$.
\end{proposition}
\begin{proof}
The piecewise affine path $\bar X^{i,n}$ is differentiable over each interval $]t_k,t_{k+1}[$ with
\begin{align*}
\frac{\d}{\d t} \bar X^{i,n}_t &= b_n(\hat \mu^n_k,X^{i,n}_k) + \frac{1}{\tau_n}\gamma^{i,n}_{k+1}\\
&= b(\bar \mu^n_t, \bar X^{i,n}_t) +(b(\hat \mu^n_k,X^{i,n}_k)-b(\bar \mu^n_t, \bar X^{i,n}_t))+(b_n(\hat \mu^n_k,X^{i,n}_k)-b(\hat \mu^n_k,  X^{i,n}_k)) + \frac{1}{\tau_n}\gamma^{i,n}_{k+1}.
\end{align*}
Now consider $(Y^1,Y^2,\dots)$ a sequence of independent copies of $Y$ which is such that $Y^i_0=\xi^i$ and let us define
\begin{align}\label{eq:proof-ii-indicators}
e^{i,n}_t = \bar X^{i,n}_t - Y^i_t, && h_t^n = \frac1n \sum_{i=1}^n \Vert e_t^{i,n}\Vert_2=\frac1n \sum_{i=1}^n \Vert \bar X^{i,n}_t-Y^i_t\Vert_2.% \Vert \bm f(\hat \nu^n_t)-\bm f(\mu_t)\Vert_2.
\end{align}
We recall that $\rho_t=\mathrm{Law}(Y_t)$ and $\hat \rho^n_t = \frac1n \sum_{i=1}^n \delta_{Y^i_t}$. Clearly, $t\mapsto e^{i,n}_t$ is differentiable over each interval $]t_k,t_{k+1}[$ with
\begin{multline*}
\frac{\d }{\d t} e^{i,n}_t =  \underbrace{(b_n(\hat \mu^n_k,X^{i,n}_k)-b(\hat \mu^n_k,  X^{i,n}_k))}_{o(1)}
\\
+ \underbrace{\big(b(\hat \mu_k,X^{i,n}_k)-b(\bar \mu^n_t, \bar X^{i,n}_t)\big)}_{I^{i,n}_t}
+ \underbrace{\big(b(\bar \mu_t^n,\bar X^{i,n}_t)-b(\rho_t, Y^{i}_t)\big)}_{J^{i,n}_t} +\frac{1}{\tau_n}\gamma^{i,n}_{k+1}.
\end{multline*}
Since $e^{i,n}_t$ is absolutely continuous and $e^{i,n}_0=0$, it holds $e^{i,n}_t=\int_0^t \Big(\frac{\d }{\d s}e^{i,n}_s\Big)\d s$ and thus
\begin{align}\label{eq:proof-ii-error}
\Vert e^{i,n}_t\Vert_2 \leq \Big\Vert\int_0^t  \frac{\d}{\d s} e^{i,n}_s\d s \Big\Vert_2\leq o(1)+ \int_0^t \left(\Vert I_s^{i,n}\Vert_2 + \Vert J_s^{i,n}\Vert_2\right)\d s + \sup_{s\leq t}\Vert M_s^{i,n}\Vert_2
\end{align}
where we have defined $M_t^{i,n} = \sum_{k=1}^{\lfloor t/\tau_n\rfloor+1} \gamma^{i,n}_{k}$. Note that we make these computations over a bounded time domain $[0,T]$, hence we write $to(1)=o(1)$. The last term comes from bounding the integral of the martingale term which is 
$$
\sum_{k=0}^{\lfloor t/\tau_n\rfloor-1} \int_{t_k}^{t_{k+1}}\frac{1}{\tau_n}\gamma^{i,n}_{k+1}\d s+ \int_{\lfloor t/\tau_n\rfloor\tau_n}^t\frac{1}{\tau_n}\gamma^{i,n}_{\lfloor t/\tau_n\rfloor+1}\d s \in [M_{t-{\tau_n}}^{i,n}, M_{t}^{i,n}]
$$
and the last inclusion is understood as inclusion in a line segment in $\RR^p$. Let us now derive estimates for $I^{i,n}_s$ and $J^{i,n}_s$. It holds for $t\in [t_k,t_{k+1}[$
\begin{align*}
\Vert I_t^{i,n}\Vert_2 &= \Vert b(\hat \mu^n_k,X^{i,n}_k)-b(\bar \mu^n_t, \bar X^{i,n}_t)\Vert_2
\leq C\Big(W_1(\hat \mu^n_k,\bar \mu^n_t) +\Vert X_{k}^{i,n}-\bar X_t^{i,n}\Vert_2\Big)
\end{align*}
and using our assumption that $b$ is uniformly bounded (and therefore $b_n$ as well),
\begin{align*}
\Vert X^{i,n}_k-\bar X^{i,n}_t\Vert_2 \leq \tau_n\Vert  b_n(\hat \mu^n_k,X^{i,n}_k)\Vert_2 +\Vert \gamma^{i,n}_{k+1}\Vert_2=O(\tau_n)+\Vert \gamma^{i,n}_{k+1}\Vert_2
\end{align*}
and
\begin{align*}
W_1(\hat \mu^n_k,\bar \mu^n_t) \leq \frac1n \sum_{i=1}^n\Vert X^{i,n}_{k}-\bar X^{i,n}_t\Vert_2\leq O(\tau_n)+\frac1n \sum_{i=1}^n\Vert \gamma^{i,n}_{k+1}\Vert_2.
\end{align*}
Plugging these in the previous estimate yields
\begin{align*}
\Vert I_t^{i,n}\Vert_2 &\leq C\Big( O(\tau_n)+\Vert \gamma^{i,n}_{k+1}\Vert_2+\frac{1}{n}\sum_{j=1}^n \Vert \gamma^{j,n}_{k+1}\Vert_2\Big)
\end{align*}
therefore
\begin{align*}
\frac{1}{n} \sum_{i=1}^n \int_0^{t} \Vert I_s^{i,n}\Vert_2\d s &= O(\tau_n)+ 2C \sum_{k=1}^{\lfloor t/\tau_n\rfloor+1}\frac1n \sum_{i=1}^n \tau_n \Vert \gamma_k^{i,n}\Vert_2.
\end{align*}
Regarding the second error term $J^{i,n}_t$, we have
\begin{align*}
\Vert J_t^{i,n}\Vert_2 &= \Vert  b(\bar \mu^n_t,\bar X^{i,n}_t)-b(\rho_t, Y^{i}_t) \Vert_2\\
&\leq C \Big(W_1(\bar \mu^n_t, \hat \rho^n_t)+ W_1(\hat \rho^n_t, \rho_t) +\Vert \bar X_{t}^{i,n}- Y_t^{i}\Vert_2\Big)\\
&\leq C (h_t^n + W_1(\hat \rho^n_t,\rho_t) +\Vert \bar X_{t}^{i,n}- Y_t^{i}\Vert_2 )%\leq O(1)(h_t^n + g_t^n)+o(1)
\end{align*}
where we have used the quantity $h^n_t$ defined in~\eqref{eq:proof-ii-indicators}. Returning to~\eqref{eq:proof-ii-error} with the estimates collected so far yields
\begin{multline}\label{eq:proof-ii-e}
\Vert e^{i,n}_t\Vert_2 
\leq o(1)+ C\int_0^t \Vert e^{i,n}_s\Vert_2 \d s + C\int_0^t h^n_s\d s + C\int_0^tW_1(\hat \rho^n_s,\rho_s)\d s  \\
\qquad + O(\tau_n) + C \sum_{k=1}^{\lfloor t/\tau_n\rfloor+1}\tau_n \left(\Vert \gamma_k^{i,n}\Vert_2+ \frac1n \sum_{j=1}^n \Vert \gamma_k^{j,n}\Vert_2\right)+  \sup_{s\leq t} \Vert M^{i,n}_s\Vert_2.
\end{multline}
By averaging over all particles we get for any $t\in [0,T]$
\begin{align*}%\label{eq:proof-ii-h}
h_t^n 
\leq 2C\int_0^t h^n_s\d s + g_t^n
\end{align*}
where
\begin{multline*}
g_t^n = o(1)+O(\tau_n)+ C\int_0^tW_1(\hat \rho^n_s,\rho_s)\d s  \\
\qquad + 2C \sum_{k=1}^{\lfloor t/\tau_n\rfloor+1}\frac1n \sum_{i=1}^n \tau_n \Vert \gamma_k^{i,n}\Vert_2+ \frac1n \sum_{i=1}^n \sup_{s\leq t} \Vert M^{i,n}_s\Vert_2.
\end{multline*}
Here we may choose $t\mapsto g_t^n$ to be nonnegative and nondecreasing. Therefore by Gr\"onwall's lemma integral form (recalled in Lemma~\ref{lem:gronwall-integral}) it holds
\begin{align}\label{eq:proof-ii-gronwall}
\sup_{t\leq T} h_t^n \leq e^{2CT} g_T^n
\end{align}
We now show that $\E[ g_T^n] \to 0$. For this purpose, let investigate the terms involving $W_1(\hat \rho^n_s,\rho_s)$, $\gamma^{i,n}_k$ and $M_s^{i,n}$ in its definition. By Lemma~\ref{lem:varadarajan}, it holds $\E[W_1(\frac1n \sum_{i=1}^n \delta_{Y^i},\mathrm{Law}(Y))]\to 0$  in $\Pp_1(\Cc([0,T];\RR^p))$ endowed with the distance $\Vert x-x'\Vert_T = \sup_{s\leq T} \Vert x(s)-x'(s)\Vert_2$. It follows (using~\cite[Eq.~(3.7)]{lacker2018mean} to upper-bound the supremum of marginal $W_1$ distances by the $W_1$ distance in path-space):
\begin{align}\label{eq:proof-ii-W1}
\E \Big[\int_0^T W_1(\hat \rho^n_s,\rho_s)\d s \Big]\leq T \E\Big[\sup_{s\leq T} W_1(\hat \rho_s^n,\rho_s) \Big] \leq T\E\Big[W_1\Big(\frac1n \sum_{i=1}^n \delta_{Y^i},\mathrm{Law}(Y)\Big)\Big] \xrightarrow[n\to \infty]{} 0.
\end{align}
For the martingale term, we have by Doob's inequality (recalled in Lemma~\ref{lem:doob} below)
\begin{align*}
\E \Big[\max_{s\leq T} \Vert M^{i,n}_s\Vert_2\Big]&\overset{\text{Jensen}}{\leq}  \Big(\E \Big[\max_{s\leq T} \Vert M^{i,n}_s\Vert^2_2\Big]\Big)^{\frac12} \\&\overset{\text{Lemma~\ref{lem:doob}}}{\leq} 2 \left(\E \Big[\sum_{1\leq k\leq \lfloor T/\tau_n\rfloor+1} \Vert \gamma^{i,n}_k\Vert_2^2\Big]\right)^{\frac12} \\
&\overset{\text{assumpt.}}{=}  2 \left( \sum_{1\leq k\leq \lfloor T/\tau_n\rfloor+1} o(\tau_n) \right)^{\frac12}  = o(1) .%o(\sqrt{t}).
\end{align*}
Finally, we have
$$
\E \Big[\sum_{k=1}^{\lfloor T/\tau_n\rfloor+1} \tau_n \Vert \gamma_k^{i,n}\Vert_2\Big] \leq \sum_{k=1}^{\lfloor T/\tau_n\rfloor+1} \tau_n \Big(\E[\Vert \gamma_k^{i,n}\Vert_2^2]\Big)^{\frac12}\leq \frac{T}{\tau_n}\tau_n\sqrt{o(\tau_n)} = o(\tau_n^{1/2}).
$$
All in all, and using our assumption that $\tau_n=o(1)$, we have shown that $\E[g^n_T] \to 0$.  Plugging into~\eqref{eq:proof-ii-gronwall}, this gives
$$
\E [\sup_{s\leq T} h^n_s] \leq e^{2C}\E[g^n_T] \xrightarrow[n\to \infty]{} 0
$$
By the triangle inequality and~\eqref{eq:proof-ii-W1}, it follows
$$
\E[ \sup_{t\in[0,T]}W_1(\bar\mu^n_t,\rho_t)] \leq \E[\sup_{t\in[0,T]} W_1(\bar\mu^n_t,\hat \rho^n_t) + \sup_{t\in[0,T]}  W_1(\hat \rho^n_t,\rho_t)] = \E [\sup_{t\in[0,T]}  h^n_t] + o(1) \xrightarrow[n\to \infty]{} 0.
$$
This proves the first claim.

As for the second claim of convergence in distribution in path space, we have, using~\eqref{eq:proof-ii-e} and the error estimates that follow:
\begin{align*}
\Vert e^{i,n}_t\Vert_2\leq C\int_0^t \Vert e^{i,n}_s\Vert_2\d s + \tilde g^{i,n}_t
\end{align*}
for some nonnegative $\tilde g^{i,n}_t$ that satisfies $\E[\tilde g^{i,n}_T]\to 0$. Therefore, by Gr\"onwall's lemma in integral form (Lemma~\ref{lem:gronwall-integral}) it holds
\begin{align*}
\E[ \sup_{t\leq T}\Vert e^{i,n}_t\Vert_2] \leq e^{Ct}\E[\tilde g^{i,n}_T] \xrightarrow[n\to \infty]{} 0.
\end{align*}
For a fixed $m\in \NN^*$, this implies 
 $$
\E \big[\max_{i=[1:m]} \sup_{s\leq T}\Vert \bar X^{i,n}_s-Y^i_s\Vert_2\big]\leq \sum_{i=1}^m \E[\sup_{s\leq T}\Vert e_s^{i,n}\Vert_2] \xrightarrow[n\to \infty]{} 0.
$$
This shows that $(\bar X^{i,n},\dots,\bar X^{m,n})$ converges in $L^1$ to $(Y^1,\dots,Y^m)$ and the convergence in distribution follows.
\end{proof} 

For convenience, we recall below two classical results which are used in the proof: Doob's inequality (see, e.g.~\cite[Theorem~1.1]{acciaio2013trajectorial}) and Gr\"onwall's lemma.
\begin{lemma}[Doob's inequality]\label{lem:doob}
Let $(M_k,\Ff_k)_{k\geq 1}$ be a real-valued martingale. Then for any $k\in \NN^*$, 
$$
\E \Big[\max_{1\leq \ell \leq k} \Vert M_\ell \Vert_2^2 \Big] \leq 4 \E \Vert M_k\Vert_2^2  = 4 \E \Big[\sum_{\ell=2}^k \Vert M_{\ell}-M_{\ell-1}\Vert_2^2 \Big].
$$
\end{lemma}

\begin{lemma}[Gr\"onwall lemma (integral form)]\label{lem:gronwall-integral}
Suppose that $h$ and $g$ are nonnegative functions on $[0,T]$ with $h$ continuous and $g$ nondecreasing, and $C>0$. If for all $t\in[0,T]$, $h(t)\leq g(t)+C\int_0^t h(s)\d s$ then for all $t\in [0,T]$, $h(t)\leq e^{Ct}\alpha(t)$.
\end{lemma}

\subsection{Case~\ref{case:jump} : the continuous-time jump process limit}

In this section we assume that $\tau_n\to 0$, that $ \alpha_n \to \alpha>0$ and that $\beta_n \to 0$. Because $\alpha_n$ (representing the average elapsed time between two updates of a given neuron) does not converge to zero, we get a continuous-time jump process  in the limit.

\subsubsection{The limit process: definition and well-posedness}

Let $\Delta T_1,\Delta T_2,\dots$ be independent exponential random variables of expectation $\alpha$ (usually denoted $\mathrm{Exp}(1/\alpha)$ where $1/\alpha$ is the \emph{rate}), and let
$$
T_0=0,\qquad T_j = \sum_{\ell=1}^j \Delta T_\ell,\qquad N_t = \max\{j\;;\; T_j\leq t\}.
$$
Remark that $T_j$ follows a $\mathrm{Gamma}(j,1/\alpha)$ distribution  and $N_t$ follows a $\mathrm{Poisson}(t/\alpha)$ distribution (and more precisely, $t\mapsto N_t$ is a Poisson process). 
We consider $Y$ the $\RR^p$-valued continuous-time stochastic process that solves $Y_0\sim \mu_0$ and
\begin{align}\label{eq:process-iii-limit}
Y_t = Y_0 -\alpha \sum_{\ell=1}^{N_t} \nabla V[\rho_{(T_\ell)^-}](Y_{(T_\ell)^-}), \qquad \rho_t =\text{Law}(Y_t)
\end{align}
where $Y_{t^-}$ denotes the left limit of $Y$ at $t$ and analogously for $\rho_{(T_\ell)^-}$.
Here $T_j$ represents the time at which the particle makes its $j$-th jump and $N_t$ represents the total number of jumps before time $t$. It is not obvious at first that this self-referential equation has a solution and whether the solution is unique, so we first study its well-posedness. We recall the notation $\Dd(\RR_+;\RR^p)$ for càdlag paths in $\RR^p$.

\begin{proposition}\label{prop:posed-iii-sde}
There exists a unique couple $(Y,\rho)$ with $Y$ a $\Dd(\RR_+;\RR^p)$-valued stochastic process and $\rho\in \Cc(\RR_+;\Pp_1(\RR^p))$ that solves~\eqref{eq:process-iii-limit}. For any $T>0$, it holds $\mathrm{Law}(Y)\in \Pp_1(\Dd([0,T];\RR^p))$. Moreover, there exists $L>0$ such that for all $0 \leq s,t\leq T$, $\E[\Vert Y_t-Y_s\Vert_2]\leq L\Vert t-s\Vert_2$. In particular, $t\mapsto \rho_t$ is $L$-Lipschitz continuous for $W_1$.
\end{proposition}
\begin{proof}
For $t\geq 0$, consider the space $\Ss_t=\Cc([0,t];\Pp_1(\RR^p))$ of continuous functions from $[0,t]$ to the Wasserstein-$1$ space $\Pp_1(\RR^p)$. Endowed with the distance $d_t(\mu,\mu')=\sup_{s\in [0,t]} W_1(\mu_s,\mu'_s)$ the space $\Ss_t$ is a complete metric space (because $(\Pp_1(\RR^p), W_1)$ itself is complete).

Given $t>0$ and the random variables $(\Delta T_k)$ and $Y_0$,
notice that for any fixed $\mu \in \Ss_t$, there is easily existence of $(Y^\mu)$, the unique $\Dd([0,t];\RR^p)$-valued stochastic process that is a strong solution of  
$$
Y^\mu_t = Y_0 -\alpha \sum_{\ell=1}^{N_t} \nabla V[\mu_{T_\ell^-}](Y^\mu_{T_\ell^-}) = Y_0 -\alpha \sum_{\ell=1}^{N_t} \nabla V[\mu_{T_\ell}](Y^\mu_{T_{\ell-1}})
$$
where we have replaced $\mu_{T_\ell^-}$ by $\mu_{T_\ell}$ in the last equation using the continuity of $s\mapsto \nabla V[\mu_s]$ for $\mu \in \Ss_T$, and we have replaced $Y^\mu_{T_{\ell}^-}$ by $Y^\mu_{T_{\ell-1}}$ which are equal, since $Y^\mu$ is constant on the intervals $[T_i,T_{i+1}[$.
The map $\Phi:\Ss_t\to \Ss_t$ defined by $\Phi(\mu)=(\text{Law}(Y^\mu_s))_{s\in [0,t]}$ is well-posed since, for $s\leq t$,
\begin{align}\label{eq:proof-iii-lip}
\E\Vert Y^\mu_t-Y^\mu_s\Vert_2 =\alpha \E \Big\Vert \sum_{\ell=N_s+1}^{N_t}\nabla V[\mu_{T_\ell}](Y^\mu_{T_{\ell-1}}) \Big\Vert_2 = O(1)\E[|N_s-N_t|] = O(1)|t-s| .
\end{align}
We are looking to prove existence and uniqueness of a fixed point for $\Phi$, at least if $t$ is small enough.

Let $\mu,\nu\in \Ss_t$. Our first step is to control the distance between the corresponding strong solutions $Y^\mu$ and $Y^\nu$ in terms of the distance between $\mu$ and $\nu$. On the one hand, it holds
\begin{align*}
\Vert Y^\mu_t - Y^\nu_t\Vert_2 &= \alpha \left\Vert \sum_{\ell=1}^{N_t} \nabla V[\mu_{T_\ell}](Y^\mu_{T_{\ell-1}}) - \sum_{\ell=1}^{N_t} \nabla V[\nu_{T_\ell}](Y^\nu_{T_{\ell-1}})\right\Vert_2\\
&\leq C' \sum_{\ell=1}^{N_t} \left(W_1( \mu_{T_\ell},\nu_{T_\ell})+ \Vert Y^\mu_{T_{\ell-1}}-Y^\nu_{T_{\ell-1}}\Vert_2\right)
\end{align*}
where $C'=\alpha C$ and $C$ is the constant from~\eqref{eq:field-regularity}. 

On the other hand, we have by recursion that
\begin{align*}   %\label{temp-eq}
\begin{split}
\Vert Y_{T_{\ell}}^\mu- Y_{T_{\ell}}^\nu \Vert_2 &\leq  \Vert Y_{T_{\ell-1}}^\mu- Y_{T_{\ell-1}}^\nu\Vert_2+ \alpha \Vert \nabla V[\mu_{T_\ell}](Y^\mu_{T_{\ell-1}}) -\nabla V[\nu_{T_\ell}](Y^\mu_{T_{\ell-1}})\Vert_2\\
&\leq \Vert Y_{T_{\ell-1}}^\mu- Y_{T_{\ell-1}}^\nu\Vert_2 +C'\Big(W_1(\mu_{T_\ell},\nu_{T_\ell})+ \Vert Y_{T_{\ell-1}}^\mu- Y_{T_{\ell-1}}^\nu\Vert_2\Big)\\
&\leq \sum_{k=1}^{\ell} C'(1+C')^{\ell-k}W_1(\mu_{T_k},\nu_{T_k}).
\end{split}
\end{align*}

Since $\Vert Y^\mu_t - Y^\nu_t\Vert_2 = \Vert Y^\mu_{T_{N_t}} - Y^\nu_{T_{N_t}}\Vert_2$ because $Y^\mu$ and $Y^\nu$ are constant on the intervals $[T_i, T_{i+1}[$, this recursion gives
\begin{align*}
\sup_{s\in [0,t]} \Vert Y^\mu_s - Y^\nu_s\Vert_2
&\leq C'(1+C')^{N_t-1}  \sum_{k=1}^{N_t} W_1(\mu_{T_k},\nu_{T_k}).
\end{align*}
Now by the law of total expectation, it holds
\begin{align*}%\label{eq:proof-iii-SDE-unique}
\E [\sup_{s\in [0,t]}\Vert Y^\mu_s - Y^\nu_s\Vert_2] \leq C' \sum_{N=0}^{+\infty} \P[N_t=N] (1+C')^{N-1}  \E\Big[\sum_{k=1}^{N} W_1(\mu_{T_k},\nu_{T_k})\mid N_t=N\Big].
\end{align*}
We recall that $N_t$ follows a Poisson distribution with intensity $t/\alpha$. By the memoryless property of the Poisson distribution, conditioned on $N_t=N$, the random variables $(T_1,\dots,T_N)$ are just a re-ordering (in increasing order) of $N$ independent uniform samples on $[0,t]$. Hence
\begin{align*}
\E [\sup_{s\in [0,t]} \Vert Y^\mu_s - Y^\nu_s\Vert_2] &\leq C' \sum_{N=0}^{+\infty} \P[N_t=N] (1+C')^{N-1}  \frac{N}{t}\int_{0}^t W_1(\mu_s,\nu_s)\d s\\
&=\frac{C'}{t}e^{-t/\alpha}\left(\sum_{N=0}^{+\infty} \frac{(t/\alpha)^N}{N!} (1+C')^{N-1} N\right) \int_{0}^t W_1(\mu_s,\nu_s)\d s\\
&=\frac{C't}{\alpha t}e^{tC'/\alpha} \int_{0}^t W_1(\mu_s,\nu_s)\d s%=C'' \int_{0}^t W_1(\mu_s,\nu_s)\d s.
\end{align*}

Since $(Y^\mu_t, Y^\nu_t)$ is a coupling of $\Phi(\mu)_t$ and $\Phi(\nu)_t$,
it follows
\begin{align*}
d_t(\Phi(\mu),\Phi(\nu))\leq \frac{C'}{\alpha }e^{tC'/\alpha}  \int_{0}^t W_1(\mu_s,\nu_s)\d s\leq \frac{C't}{\alpha }e^{tC'/\alpha}  d_t(\mu,\nu) \, .
\end{align*}
Hence by taking $t_0$ small enough, the map $\Phi$ is a contraction on $\Ss_{t_0}$. By Banach's fixed point theorem, $\Phi$ has a unique fixed-point, which proves existence and uniqueness of solutions to~\eqref{eq:process-iii-limit} on a time interval $[0,t_0]$.
 Since the constant $C'$ above does not depend on the initial distribution $\mu_0$, we can repeat the same argument starting from $t_0$ and recursively, we get existence and uniqueness   over any time interval $[0,t]$ for any $t>0$ and hence, on $\RR_+$.
 
Then, from~\eqref{eq:proof-iii-lip} we deduce that there exists $L>0$ such that for $0\leq s,t\leq T$  $\E[\Vert Y_t-Y_s\Vert_2]\leq L|t-s|$. Since $W_1(\mu_s,\mu_t)\leq \E \Vert Y_s-Y_t\Vert_2$, this implies that $t\mapsto \mu_t$ is $L$-Lipschitz continuous for $W_1$.
\end{proof}

\subsubsection{Evolution equation in the space of probability measures}
In this subsection, we study the equation satisfied by $\rho_t = \mathrm{Law}(Y_t)$ which, as will be proven, is given by the following evolution equation
\begin{align}\label{eq:iii-PDE}
\partial_t \rho_t = \frac{1}{\alpha}((\Id-\alpha \nabla V[\rho_t])_\# \rho_t-\rho_t),&& \rho_0=\mu_0.
\end{align}
We say that $\rho$ solves this equation if $\rho \in \Cc(\RR_+;\Pp_1(\RR^p))$ %(more specifically if $t\mapsto \rho_t$ is absolutely continuous in $(\Pp_1(\RR^p),W_1)$) 
and~\eqref{eq:iii-PDE} is satisfied in the weak sense, that is
\begin{align*}
\frac{\d}{\d t} \int \psi(x)\d\rho_t(x) = \frac{1}{\alpha} \int \Big( \psi(x-\alpha \nabla V[\rho_t](x))-\psi(x)\Big)\d\rho_t(x), &&\forall \psi\in \Cc^b(\RR^p)
\end{align*}
where $\Cc^b(\RR^p)$ is the set of bounded and continuous functions on $\RR^p$.
Here again, the fact that this equation is well-posed is not immediate. We first prove that the law of $(Y_t)$ where $Y$ the solution to~\eqref{eq:process-iii-limit} is a solution -- thereby proving existence of solutions -- and separately, we will prove uniqueness of solutions below.

\begin{proposition} If $(Y,\rho)$ is the unique solution to~\eqref{eq:process-iii-limit}, then $\rho$ is a solution to~\eqref{eq:iii-PDE}.
\end{proposition}
\begin{proof}
Let $\psi\in \Cc^{b}(\RR^p)$ be a test function. For any $t,\Delta t>0$, by partitioning the expectation according to the number of jumps $\Delta N=N_{t+\Delta t}-N_t$ between $t$ and $t+\Delta t$ (which is a Poisson random variable of intensity $\Delta t/\alpha$) we have{, by continuity of $t \mapsto \nabla V[\rho_t]$,}
\begin{multline*}
\E [\psi(Y_{t+\Delta t})] = \sum_{N=0}^\infty \P(\Delta N=N) \E\Big[\psi\Big(Y_t-\alpha \sum_{\ell=1}^N \nabla V[\rho_{T_{N_t+\ell}}](Y_{T_{N_t+\ell}^-})\Big)\mid \Delta N =N\Big] \\
= e^{-\Delta t/\alpha}\E\Big[\psi(Y_t)\Big]+\frac{\Delta t}{\alpha}e^{-\Delta t/\alpha} \E\Big[\psi(Y_t-\alpha \nabla V[\rho_{T_{N_t+1}}](Y_{T_{N_t+1}^-}))\mid \Delta N= 1\Big]\\ 
+\sum_{N=2}^\infty \frac{(\Delta t/\alpha)^N}{N!}e^{-\Delta t/\alpha} \E\Big[\psi\Big(Y_t-\alpha \sum_{\ell=1}^N \nabla V[\rho_{T_{N_t+\ell}}](Y_{T_{N_t+\ell}^-})\Big)\mid \Delta N=N\Big].
\end{multline*}
Using that $Y_{T_{N_t+1}^-}=Y_t$, it follows
\begin{multline*}
 \frac{\E [\psi(Y_{t+\Delta t})] -\E [\psi(Y_t)] }{\Delta t} \\= \frac{e^{-\Delta t/\alpha}-1}{\Delta t}\E \Big[\psi(Y_t)\Big]+ \frac{1}{\alpha}e^{-\Delta t/\alpha} \E\Big[\psi(Y_t-\alpha \nabla V[\rho_{T_{N_t+1}}](Y_t))\mid \Delta N=1\Big]\\
 \quad +\frac{1}{\Delta t}\sum_{N=2}^\infty \frac{(\Delta t/\alpha)^N}{N!}e^{-\Delta t/\alpha} \E\Big[\psi\Big(Y_t-\alpha \sum_{\ell=1}^N \nabla V[\rho_{T_{N_t+\ell}}]({Y_{T_{N_t+\ell}^-}})\Big)\mid \Delta N=N\Big].
\end{multline*}
We then get
\[
\frac{\E [\psi(Y_{t+\Delta t})] -\E [\psi(Y_t)] }{\Delta t}  \xrightarrow[\Delta t\to 0]{} \frac{-1}{\alpha}\E [\psi(Y_t)]+\frac{1}{\alpha} \E\Big[\psi(Y_t-\alpha \nabla V[\rho_t](Y_t))\Big]+0 ,
\]
by the continuity of $t\mapsto \nabla V[\rho_t]$ in the supremum norm for the second term (indeed, if $\Delta N=1$ then $T_{N_1+1}\in [t,t+\Delta t]$), and the boundedness of $\psi$ to obtain that the last expression converges to $0$.
This shows that $t\mapsto \E[\psi(Y_t)]$ is differentiable, with derivative
\begin{align*}
\frac{\d}{\d t} \E[\psi(Y_t)] =\frac{1}{\alpha} \E[\psi(Y_t-\alpha \nabla V[\rho_t](Y_t))-\psi(Y_t)]
\end{align*}
which, complemented with the initial condition $\E[\psi(Y_0)]=\int \psi\d\mu_0$, shows that $\rho$ is a weak solution to~\eqref{eq:iii-PDE}.
\end{proof}
Regarding uniqueness of solutions of the evolution equation~\eqref{eq:iii-PDE}, we will use a fixed-point argument that directly proves existence and uniqueness, independently of the stochastic process solving~\eqref{eq:process-iii-limit}.
\begin{proposition}\label{prop:existence-uniqueness-alt}
Given $\mu_0\in \Pp_1(\RR^p)$, there exists a unique $\rho \in \Cc(\RR_+;\Pp_1(\RR^p))$ that solves~\eqref{eq:iii-PDE}. It is also the unique solution to the integral equation:
\begin{align*}%\label{eq:jump-integral-form}
\rho_t = e^{-t/\alpha} \mu_0 +\frac{1}{\alpha} \int_0^t e^{(s-t)/\alpha}(\Id-\alpha \nabla V[\rho_s])_\# \rho_s\d s.
\end{align*}
\end{proposition}
\begin{proof}
Consider the operator $\Phi:\Cc(\RR_+;\Pp_1(\RR^p))\to \Cc(\RR_+;\Pp_1(\RR^p))$ defined as
$$
\Phi(\rho)_t =  e^{-t/\alpha} \mu_0 +\frac{1}{\alpha} \int_0^t e^{(s-t)/\alpha}(\Id-\alpha \nabla V[\rho_s])_\# \rho_s\d s.
$$
Let us check that this operator indeed maps into $\Cc(\RR_+;\Pp_1(\RR^p))$. First, the right-hand side of this equation should be understood in duality with bounded continuous functions, that is $\Phi(\rho)_t$ is characterized by
$$
\int_{\RR^p}  \psi \d \Phi(\rho)_t  = e^{-t/\alpha} \int_{\RR^p}  \psi \d \mu_0 + \frac{1}{\alpha} \int_0^t e^{(s-t)/\alpha}\int_{\RR^p} \psi(x-\alpha \nabla V[\rho_s](x))\d\rho_s(x)\d s.
$$
Since the right-hand side of this equation defines a continuous linear map over $\Cc^b(\RR^p)$ that is nonnegative and maps the constant function equal to $1$ to $1$, this indeed defines a probability measure. Moreover, for any $\psi\in \Cc^b(\RR^p)$, the map $t\mapsto \int \psi \Phi(\rho)_t$ is differentiable and 
\begin{align*}
\frac{\d}{\d t} \int \psi\d\Phi(\rho)_t 
&= -\frac{e^{-t/\alpha}}{\alpha}\int \psi \d\mu_0 -\frac{1}{\alpha^2} \int_0^t e^{(s-t)/\alpha}\int \psi(x-\alpha \nabla V[\rho_s](x))\d\rho_s(x)\d s\\
&\qquad +\frac{1}{\alpha} \int \psi(x-\alpha \nabla V[\rho_t](x))\d\rho_t(x)\\
&= \frac{1}{\alpha} \left[ \int \psi(x-\alpha \nabla V[\rho_t](x))\d\rho_t(x) -\int \psi\d\Phi(\rho_t)\right].
\end{align*}
%\LC{new argument (following Pierre's comment)} 
Clearly, if $\Phi(\rho)=\rho$ then $\rho$ solves~\eqref{eq:iii-PDE}. On the other hand assume that $\rho$ solves~\eqref{eq:iii-PDE}; then by the previous computation, for $\psi \in \Cc^b(\RR^p)$ it holds for almost every $t\in \RR_+$,
$$
\frac{\d}{\d t} \int \psi\d\Phi(\rho)_t  = \frac{\d}{\d t} \int \psi\d\rho_t + \frac{1}{\alpha} \int \psi\d\rho_t -\frac{1}{\alpha }\int \psi\d\Phi(\rho)_t.
$$
So if we let $h(t)=\int \psi \d(\Phi(\rho)_t-\rho_t)$ this function is absolutely continuous with $h(0)=0$ and for almost every $t\in \RR_+$
$$
h'(t) = -\frac{1}{\alpha} h(t)
$$
and so $h(t)=0$. This implies that $\rho$ is a fixed point of $\Phi$. So $\rho$ solves~\eqref{eq:iii-PDE} iff $\Phi(\rho)=\rho$. We are thus looking for fixed points of $\Phi$.

Recall that $\Ss_t = \Cc([0,t];\Pp_1(\RR^p))$ is endowed with the metric $d_t(\rho,\rho')=\sup_{s\in [0,t]} W_1(\rho_s,\rho'_s)$. We will next show that $\Phi$ is a contraction of the complete metric space $(\Ss_t,d_t)$ for $t$ small enough. By the subadditivity of $W_1$ (\cite[Theorem~4.8]{villani2008optimal}), we have
\begin{align*}
W_1(\Phi(\rho)_t,\Phi(\rho')_t) &\leq  \frac{1}{\alpha} \int_0^t e^{(s-t)/\alpha}W_1((I-\alpha \nabla V[\rho_s])_\# \rho_s,(I-\alpha \nabla V[\rho'_s])_\# \rho'_s)\d s.
\end{align*}
Now for all Lipschitz maps $T,T':\RR^p\to \RR^p$ it is not difficult to show that
\begin{align*}
W_1((\Id+T)_\#\rho,&(\Id+T')_\# \rho') \\
&\leq W_1((\Id+T)_\#\rho,(\Id+T)_\# \rho')+W_1((\Id+T)_\#\rho',(\Id+T')_\# \rho')\\
&\leq \text{Lip}(T)W_1(\rho,\rho')+\Vert T-T'\Vert_{L^1(\rho')}.
\end{align*}
Coming back to our previous formula, we deduce
\begin{align*}
W_1(\Phi(\rho)_t,\Phi(\rho')_t) 
&\leq  \frac{1}{\alpha} \int_0^t e^{(s-t)/\alpha}(LW_1(\rho_s,\rho'_s)+\alpha\Vert \nabla V[\rho_s]-\nabla V[\rho'_s]\Vert_\infty)\d s
\end{align*} 
where $L$ here refers to a bound on the Lipschitz constant of any map of the form $x\mapsto -\alpha \nabla V[\rho_s](x)$ which exists under Assumption~\ref{ass:phi}. Since the last term is also bounded by ${C} W_1(\rho_s,\rho'_s)$ {by Lemma \ref{lem:MF-Lip}}, one has that  for some ${C'}>0$, 
$$
W_1(\Phi(\rho)_t,\Phi(\rho')_t) \leq {C'} \int_0^t e^{(s-t)/\alpha}W_1(\rho_s,\rho'_s)\d s.
$$
This implies 
$$
d_t(\Phi(\rho),\Phi(\rho'))\leq {C'} \int_0^t e^{(s-t)/\alpha}W_1(\rho_s,\rho'_s)\d s \leq  {C'} \alpha(1-e^{-t/\alpha}) d_t(\rho,\rho').
$$
Hence for $t$ small enough, $\Phi$ is a contraction and admits a unique fixed point. Since ${C'}$ does not depend on $\mu_0$, we can repeat the argument and get existence and uniqueness on $\RR_+$.
\end{proof}

\subsubsection{Proof of the mean-field limit}

\paragraph{Piecewise constant interpolation.} In order to formalize the sense in which $(X^{i,n})$ converges to $Y$, let us map the discrete-time processes $(X^{i,n})$ into continuous-time ones via a piecewise constant interpolation of the trajectories. For $k\in \NN$, let $t_k = \tau k$, and for $t\in [t_k,t_{k+1}[$ let
\begin{align*}
\bar X^{i,n}_t &= X^{i,n}_{k} , && \bar \mu^n_t = \frac1n \sum_{i=1}^n \delta_{\bar X^{i,n}_t}.
\end{align*}
Claim~\ref{case:jump} of the main theorem (Theorem~\ref{thm:main}) is then a consequence of the following statement:
\begin{proposition}\label{prop:jump-limit}
Assume that $\tau_n\to 0$, that $ \alpha_n \to \alpha>0$ and that $\beta_n\to 0$. For any $t\geq 0$, it holds $\E[W_1(\hat \mu^n_{\lfloor t/\tau{_n}\rfloor},\rho_t)]\xrightarrow[n\to \infty]{}0$ where $\rho$ is the unique solution to~\eqref{eq:iii-PDE}. 
\end{proposition}

\begin{proof} For this proof, we will build a coupling between the discrete and the limit dynamics and estimate the errors in an asynchronous way. This requires first to reformulate the discrete dynamics in a form analogous to the limit one.
 \paragraph{Step 1: Reformulation of the discrete dynamics.}
 The first step of this proof is identical to the beginning of the proof of Proposition~\ref{prop:discrete-limit}. The update equation~\eqref{eq:GD-dropout-weights} can be decomposed, for $k\in \NN$, as
 \begin{align*}
  X^{i,n}_{k+1} &= X^{i,n}_{k}  -\tau_n \cdot (1+\eta^i_{k+1})D\bm{\phi}(X^{i,n}_{k})^\top \left(\frac1n \sum_{j=1}^n (1+\eta^j_{k+1}) \bm{\phi}(X^{j,n}_{k})-\bm{y}\right)\\
  &=X^{i,n}_{k}  -\tau_n \cdot (1+\eta^i_{k+1}) \left(\nabla V[\hat \mu^n_k](X^{i,n}_{k}) +\gamma^{i,n}_{k+1}\right)
 \end{align*}
 where we have gathered errors terms in $\gamma^{i,n}_k$ defined  as 
\begin{align*}
\gamma^{i,n}_k &= \frac1n \sum_{j=1}^n \eta^{j}_k D\bm \phi(X^{i,n}_{k-1})^\top \bm \phi(X^{j,n}_{k-1})
\end{align*}
We have under Assumption~\ref{ass:phi}
\begin{align}\label{eq:proof-iii-gamma}
\E \Vert\gamma^{i,n}_k\Vert_2\leq \sqrt{\E \Vert\gamma^{i,n}_k\Vert_2^2} =\sqrt{{O\Big(\frac{n\E [\eta^2]}{n^2}}\Big)} = O(\sqrt{\beta_n})\to 0,
\end{align}
which suggests that these quantities indeed are error terms, as it will be confirmed in the rest of the proof. 

Let $\Delta K^i_\ell$ be the number of steps between jump number $\ell$ and $\ell+1$ {of the $i$-th neuron}, that is if $\eta^i_k>0$ is the $\ell$-th nonnegative term in the sequence $(\eta^i_k)_k$, we set $\Delta K^i_\ell = \min\{k'\;;\; \eta^i_{k+k'}>0\}$. By construction, the variables $\Delta K^i_1,\Delta K^i_2\dots$ are independent
 $\text{Geom}(q)$ random variables, with support $\{1,2,\dots\}$. Also let
$$
K^i_0 =0,\qquad K^i_j = K_0^i+ \sum_{\ell=1}^{j} \Delta K^i_\ell, \qquad M^i_k = \max\{j\;;\; K^i_j\leq k\}
$$
Here $(K^i_j)$ for $j\geq 1$ is the index of the (landing of the) $j$-th jump and corresponds to the $j$-th time $(\eta^i_k)_k$ takes a positive value{, while $M^i_k$ is the number of jumps landing before or on step $k$}.

\paragraph{Step 2: Coupling with the limit dynamics.} Now consider $(Y^1,Y^2,\dots)$ independent copies of $Y$ (the solution to~\eqref{eq:process-iii-limit}). Let us couple these copies with the discrete dynamics as follows: (i) we take $Y^i_0=\xi^i$ and (ii) we take the coupling $(\Delta T^i_j, \Delta K^i_j)$ given by Lemma~\ref{lem:jump-coupling}. For $k\in [K^i_\ell: K^i_{\ell+1}-1]$ (hence it holds $\ell=M^i_k$), we define 
$$
e^{i,n}_k = X^{i,n}_k - Y^i_{T^i_\ell}. 
$$
Since $Y^i$ is constant on the interval $[T^i_\ell, T^{i}_{\ell+1}[$, we could have replaced in this definition $Y_{T^i_\ell}$ by $Y_t$ for any $t$ in this interval without consequences. In words, $e^{i,n}_k$ compares the coupled processes $X^{i,n}$ and $Y^i$ after the same number of jumps. Note that, \emph{in expectation}, this quantity behaves similarly as the difference between these processes \emph{at the same time} since
\begin{align}\label{eq:proof-iii-slight}
\E \Vert X^{i,n}_k - Y^i_{\tau_n k}\Vert_2 &\leq \E \Vert X^{i,n}_k-Y^i_{T^i_{M^i_k}}\Vert_2 + \E \Vert Y^i_{T^i_{M^i_k}} - Y^i_{\tau_n k}\Vert_2 \leq \E \Vert e^{i,n}_k\Vert_2 + {O(\tau_n + |\alpha_n - \alpha|)}
\end{align}
where the estimate $\E \Vert Y^i_{T^i_{M^i_k}} - Y^i_{\tau_n k}\Vert_2 {= \E \Vert Y^i_{T^i_{M^i_k}} - Y^i_{T^i_{N^i_{\tau_n k}}}\Vert_2} =O(\tau_n + |\alpha_n - \alpha|)$ is proved in Lemma~\ref{lem:jump-times-increment}. {However, the comparator $e^{i,n}_k$ will prove simpler to work with in the following.}

If $\eta^i_{k+1}=-1$, then it holds $e^{i,n}_{k+1} = e^{i,n}_k$ (there is no jump), otherwise $\eta^i_{k+1}=(1-q_n)/q_n$ and $k+1=K^i_{\ell+1}$ for $\ell = M^i_k$ so it holds
\begin{align*}
e^{i,n}_{k+1} &= e^{i,n}_k -  \frac{\tau_n}{q_n} \left(\nabla V[\hat \mu^n_k](X^{i,n}_{k}) +\gamma^{i,n}_{k+1}\right) + \alpha \nabla V[\rho_{T^i_{\ell+1}}](Y^{i}_{T^i_{\ell}})
\end{align*}
hence 
\begin{align*}
\Vert e^{i,n}_{k+1}\Vert_2 &\leq \Vert e^{i,n}_k\Vert_2  +  \alpha C \left(W_1(\hat \mu^n_k, \rho_{T^i_{\ell+1}}) + \Vert X^{i,n}_{k}- Y^{i}_{T^i_{\ell}}\Vert_2 +\Vert \gamma^{i,n}_{k+1}\Vert_2\right) \\
&\qquad + (C+\Vert \gamma^{i,n}_{k+1}\Vert_2) \vert \alpha_n -\alpha\vert.
\end{align*}
Moreover, the Wasserstein distance can be decomposed as 
\begin{align*}
W_1(\hat \mu^n_k, \rho_{T^i_{\ell+1}}) &\leq W_1(\hat \mu^n_k, \hat \rho^n_{k\tau_n}) + W_1(\hat \rho^n_{k\tau_n}, \rho_{k\tau_n})+W_1(\rho_{k\tau_n},\rho_{T^i_{\ell+1}}).
\end{align*}
Also, by the Lipschitz property of Proposition~\ref{prop:posed-iii-sde}, it holds 
$$
W_1(\rho_{k\tau_n},\rho_{T^i_{\ell+1}})\leq L|k\tau_n- T^i_{\ell+1}| \leq L\cdot |\tau_n (K^i_{\ell+1}-1)- T^i_{\ell+1}| \leq L\cdot |\tau_n K^i_{\ell+1}- T^i_{\ell+1}|  +o(1).
$$
Hence, still when $k+1=K^i_{\ell+1}$, it holds (assuming $|\alpha_n-\alpha|\leq \alpha C$, which holds eventually):
\begin{multline*}%\label{eq:proof-iii-ek}
\Vert e^{i,n}_{k+1}\Vert_2 \leq \Vert e^{i,n}_k\Vert_2  +  \alpha C \Big(W_1(\hat \mu^n_k, \hat \rho^n_{k\tau_n})+ W_1( \hat \rho^n_{k\tau_n}, \rho_{k\tau_n}) \\
\qquad \qquad + L\cdot \sup_{\ell\leq M^i_k}|\tau_n K^i_{\ell+1}- T^i_{\ell+1}| + \Vert e^{i,n}_k\Vert_2  +2\Vert \gamma^{i,n}_{k+1}\Vert_2\Big) +o(1)
\end{multline*}
Now, for a fixed (deterministic) $k$, accounting for the probability $q_n$ of jumps, it follows taking expectation and using~\eqref{eq:proof-iii-gamma} that
\begin{align*}
\E \Vert e^{i,n}_{k+1}\Vert_2 &\leq  \E \Vert e^{i,n}_k\Vert_2 \\
&\,  +  \alpha q_n  C \E\left(W_1(\hat \mu^n_k, \hat \rho^n_{k\tau_n}) + W_1( \hat \rho^n_{k\tau_n}, \rho_{k\tau_n})+L\cdot |\tau_n K^i_{\ell+1}- T^i_{\ell+1}| + \Vert e^{i,n}_k\Vert_2\right) +o(q_n)\\
&\leq  (1+\alpha Cq_n )\E \Vert e^{i,n}_k\Vert_2  +  \alpha q_n C \E W_1(\hat \mu^n_k, \hat \rho^n_{k\tau_n})   + o(q_n)
\end{align*}
where the second inequality uses Corollary~\ref{corollary:sup-jump-times-gap} (that shows $\E \sup_{\ell\leq M^i_k} |\tau_n K^i_{\ell+1}- T^i_{\ell+1}|=o(1)$) and  Lemma~\ref{lem:empirical-iii}  (that shows $\E[\sup_{s\leq t} W_1(\hat \rho^n_{s}, \rho_{s})] \xrightarrow[n\to \infty]{} 0$). We note that the term $o(q_n)$ is uniform for $k\in [0,t/\tau_n]$.

Consider now the quantity $ h^n_k= \frac1n \sum_{i=1}^n \E \Vert e^{i,n}_k\Vert_2$ which satisfies, by~\eqref{eq:proof-iii-slight}, 
\begin{align*}
\E W_1(\hat \mu^n_k, \hat \rho^n_{k\tau_n})\leq \frac1n \sum_{i=1}^n \E\Vert X^{i,n}_k - Y^i_{k\tau_n}\Vert_2\leq h^n_k +o(1).
\end{align*}
 From the previous inequality, we deduce
\begin{align*}
h^n_{k+1} \leq (1+2\alpha Cq_n )h^n_k + o(q_n).
\end{align*}
Since $h^n_0=0$, it follows that
$$
\sup_{k\leq t/\tau_n} h^n_k \leq o(q_n) \sum_{i=0}^{k} (1+2\alpha C q_n)^i \leq \frac{o(q_n)}{\tau_n}(1+2C\tau_n)^{t/\tau_n} {= o\Big(\frac{1}{\alpha_n}\Big)} \xrightarrow[n\to \infty]{} 0.
$$
This shows that for any $t\geq 0$,
\begin{align*}
\E W_1(\hat \mu^n_{\lfloor t/\tau_n\rfloor },\rho_t) &\leq \E W_1(\hat \mu^n_{\lfloor t/\tau_n\rfloor },\hat \rho^n_{\tau_n\lfloor t/\tau_n\rfloor }) + \E W_1(\hat \rho^n_{\tau_n\lfloor t/\tau_n\rfloor }, \rho_{\tau_n\lfloor t/\tau_n\rfloor })+\E W_1(\rho_{\tau_n\lfloor t/\tau_n\rfloor }, \rho_t)\\
&\leq h^n_{\lfloor t/\tau_n\rfloor}+o(1)+ L | t- \tau_n \lfloor t/\tau_n\rfloor | \xrightarrow[n\to \infty]{} 0.\qedhere
\end{align*}
\end{proof}

 \subsubsection{Auxiliary lemmas}
In the following lemma, we consider $\rho$ and $\hat\rho^n$ as defined in the proof of Proposition~\ref{prop:jump-limit}.
\begin{lemma}\label{lem:empirical-iii}
For any $t>0$, it holds $\E [ \sup_{s\leq t} W_1(\hat \rho^n_s,\rho_s)] \to 0$.
 \end{lemma}
\begin{proof}
We have seen in~\eqref{eq:proof-iii-lip} that $t\mapsto \rho_t$ is Lipschitz continuous for $W_1$, and the same holds for $t\mapsto \hat \rho^n_t$ via the same argument. Let us denote $L$ a common Lipschitz constant. For $\varepsilon>0$, we thus have 
$$
\sup_{s\leq t} W_1(\hat \rho^n_s,\rho_s) \leq \varepsilon +  \sup_{k\in [0: \lceil Lt/\varepsilon \rceil] } W_1(\hat \rho^n_{s_k},\rho_{s_k})
$$
where $s_k = k\varepsilon /L$. Taking expectations, it follows
$$
\E\Big[ \sup_{s\leq t} W_1(\hat \rho^n_s,\rho_s) \Big]\leq \varepsilon + \sum_{k=0}^{\lceil Lt/\varepsilon \rceil]} \E[ W_1(\hat \rho^n_{s_k},\rho_{s_k})] \xrightarrow[n\to \infty]{} \varepsilon
$$
by Lemma~\ref{lem:varadarajan}. Since $\varepsilon$ is arbitrary, this proves the claim. 
\end{proof}

\begin{lemma}[Jump times coupling]\label{lem:jump-coupling} Let $\tau_n,q_n>0$ be such that $\tau_n\to 0$ and $\alpha_n=\frac{\tau_n}{q_n} \to \alpha>0$. If $X\sim \mathrm{Exp}(1/\alpha)$ and $Y_n/\tau_n \sim \mathrm{Geom}(q_n)$ then there exists a sequence of couplings $(X,Y_n)$ such that $\lim_{n\to \infty}\E|Y_n-X|=0$. Moreover, for any $x>0$, these couplings satisfy $\lim_{n\to \infty} \mathrm{ess\text{-}sup} \big[|Y_n-X| \big| X\leq x\big] = 0$.
\end{lemma}
\begin{proof}
For $\tau>0$, the random variable $Z=\lceil X/\tau \rceil$ has distribution $ \mathrm{Geom}\left(1- e^{-\tau/\alpha} \right)$. Indeed, for $k = 1, 2, \dots$ we have
\begin{align*}
    P(Z \geq k) = P( X > (k-1)\tau ) = \exp(-(k-1) \tau/\alpha) = \left(e^{-\tau/\alpha}\right)^{k-1}
\end{align*}
Recall that a random variable $\zeta$ has distribution $\mathrm{Geom}(\pi)$ iff for $k=1,2,\dots$ it holds that
\begin{align*}
    P(\zeta \geq k) = \sum_{j=k}^\infty (1-\pi)^{j-1} \pi = (1-\pi)^{k-1} \pi \sum_{j=0}^\infty (1-\pi)^j = (1-\pi)^{k-1}.
\end{align*}
Hence, $Z \sim \mathrm{Geom}\left(1- e^{-\tau/\alpha} \right)$. Now consider $X^\prime_n = - X \cdot \tau_n/ (\alpha\log(1-q_n) )$ and $Y_n = \tau_n \lceil X_n^\prime/\tau_n \rceil$. By construction, $X_n^\prime \sim \mathrm{Exp}\left( 
- \log(1-q_n)/\tau_n \right)$ and $Y_n/\tau_n\sim \mathrm{Geom}(q_n)$. Moreover, we have
\begin{align*}
    |X - Y_n| &\leq  |X-X^\prime_n| +  |X^\prime_n - \tau_n \lceil X^\prime_n/\tau_n\rceil|
    \\
    &\leq \left|1-\frac{\tau_n}{-\alpha \log(1-q_n)}\right| |X| + \tau_n 
    \\
    &=o(1) |X| + \tau_n.% \xrightarrow[n\to \infty]{} 0
\end{align*}
The first claim follows by taking the expectation, the second by taking the essential-supremum conditioned on $X\leq x$.
\end{proof}

\begin{lemma}\label{lem:jump-times-gap}
% \yerkincomment{For cleaner writing we could say: Let $\alpha_n = \alpha>0$ be constant. }
If $(\Delta T_l, \Delta K_l)$ follows the coupling in Lemma \ref{lem:jump-coupling} the following inequalities hold for large enough $n$
\begin{equation*}%\label{eq:lem17-eq1}
   \Delta T_l \frac{\alpha_n}{(1+q_n)\alpha} \leq \tau_n \Delta K_l < \Delta T_l \frac{\alpha_n}{\alpha} + \tau_n.
\end{equation*}
Moreover, for any $j\leq N_T$ and large enough $n$ we have
\begin{equation*}%\label{eq:lem17-eq2}
    T_j\frac{\alpha_n}{(1+q_n)\alpha} \leq \tau_n K_j < T_j \frac{\alpha_n}{\alpha} + j\tau_n. 
\end{equation*}
In particular, for $n$ large enough, and any $k\leq T/\tau_n$,
\begin{equation*}%\label{eq:comparison-nb-jumps}
M_k = \sup\{ j\;;\; K_j\leq k\} \leq N_{2\tau_nk} = \sup\{ j\;;\; T_j\leq 2\tau_nk\}.
\end{equation*}
\end{lemma} 

\begin{proof}
From the Taylor expansion $-\log(1-q_n) = q_n + q_n^2 /2 + O(q_n^3)$ it follows that for large enough $n$ 
$$
    \frac{1}{\alpha (q_n + q_n^2)} \leq 
    \frac{1}{-\alpha \log(1-q_n)} <
    \frac{1}{\alpha q_n}.
$$
Now recalling that $\tau_n \Delta K_l = \tau_n \lceil -\Delta T_l/ (\alpha \log(1-q_n)) \rceil$ we get
\[
    \tau_n \Delta K_l \leq \tau_n \left\lceil \frac{\Delta T_l}{\alpha_n q_n} \cdot\frac{\alpha_n}{\alpha} \right\rceil \leq \tau_n \left( \frac{\Delta T_l}{\alpha_n q_n} \cdot\frac{\alpha_n}{\alpha} + 1 \right) = \tau_n +  \Delta T_l \frac{\alpha_n}{\alpha}
\]
and
\[
\tau_n \Delta K_l \geq 
\tau_n \left \lceil \frac{\Delta T_l}{\alpha (q_n + q_n^2)} \right\rceil \geq
\Delta T_l \frac{\alpha_n}{(1+q_n)\alpha}.
\]
For the last expression, we use
$$
M^n_k =\sup\{ j\;;\; \tau_n K^n_j\leq \tau_n k\}\leq  \sup\{ j\;;\; T_j\frac{\alpha_n}{(1+q_n)\alpha}\leq \tau_n k\}= N_{\tau_n k \frac{(1 + q_n) \alpha}{\alpha_n}}\leq N_{2\tau_n k}
$$
for $n$ large enough.
\end{proof}

%---------------------------------------------
\begin{corollary}\label{corollary:sup-jump-times-gap}
It holds
    $$
    \E \sup_{l\leq M_{T/\tau_n}} |\tau_n K_{l+1} - T_{l+1}| = o(1)
    .$$
\end{corollary}
\begin{proof}
By Lemma \ref{lem:jump-times-gap} we have
    $$
        |\tau_n K_{l+1} - T_{l+1}| \leq \max \left\{ T_{l+1} \Bigg| \frac{\alpha_n}{\alpha(1+q_n)}-1 \Bigg|,(l+1) \tau_n + T_{l+1}|\alpha_n/\alpha-1| 
        \right\}.
    $$
Taking supremum and expectation we get
\begin{multline*}
    \E\sup_{l\leq M_{T/\tau_n}} |\tau_n K_{l+1} - T_{l+1}| \leq
    \\
    \E\max \left\{o(1)T_{M_{\lfloor T/\tau_n \rfloor}+1},
    \tau_n (M_{\lfloor T/\tau_n \rfloor}+1) +  o(1) T_{M_{\lfloor T/\tau_n \rfloor}+1} \right\} = o(1).
\end{multline*}
The last identity holds by observing that, by Lemma \ref{lem:jump-times-gap}, for $k:= \lfloor T/\tau_n \rfloor$, $M_k \leq N_{2 \tau_n k}$, hence
\[
\E(M_{k}+1) \leq \E(N_{2 \tau_n k} + 1) = \frac{2 \tau_n k}{\alpha} + 1 = O(1),
\]
and
\[
\E (T_{M_{\lfloor T/\tau_n \rfloor}+1}) 
\leq \E (T_{N_{2 \tau_n k}+1}) \leq \frac{2 \tau_n k + 1}{\alpha} = O(1).
\]
\qedhere
\end{proof}
%-------------------------------------------------
%-------------------------------------------------
\begin{lemma}\label{lem:jump-times-increment}
We have    
$$
    \sup_{k \leq T/\tau_n} \E \| Y_{T_{M_k}} - Y_{T_{N_{\tau k}}} \| \leq {C (\tau_n + |\alpha_n - \alpha|)}
$$
{for some constant $C>0$ independent of $n$.}
\end{lemma}
\begin{proof}
% Let us first bound $P(N_{\tau k} - M_k \geq r| N_T = m)$, where $r, m$ are natural numbers and $1\leq r\leq m$.
Take some $k \leq T/\tau_n$. First we show that $|N_{\tau_n k} - M_k|$ is dominated by the number of points $T_j$ that land in a neighborhood of $\tau_n k$ of size $\tau_n O(N_T)$. Let $r$ be a natural number{, $\ell = N_{\tau_n k}$,} and observe the following sequence of implications
\begin{align*}
& N_{\tau_n k} - M_k \geq r \implies 
T_l\leq \tau_n k < T_{l+1} \text{ and } K_{l-r+1} > k
\\
& \implies T_{l-r+1}\leq\dots\leq T_l\leq \tau_n k \leq \tau_n K_{l-r+1}
\\ &\implies T_{l-r+1},\dots,T_l \in [\tau_n k - \varepsilon, \tau_n k + \varepsilon],
\end{align*}
where $\varepsilon = \max\{T | \frac{\alpha_n}{\alpha(1+q_n)}-1 |, \tau_n N_T + T|\alpha_n/\alpha-1|\}$ and the last implication uses the bound $|T_{l-r+1} -\tau_n K_{l-r+1}| \leq \varepsilon$ from Lemma \ref{lem:jump-times-gap}. 
Hence, if $N_{\tau_n k} - M_k \geq r$, then at least $r$ points among $T_1,\dots,T_{N_T}$ fall into the interval $[\tau_n k - \varepsilon, \tau_n k + \varepsilon]$. The same way we can show that $M_k - N_{\tau_n k} \geq r$ implies that at least $r$ points fall into the interval $[\tau_n k - \varepsilon, \tau_n k + \varepsilon]$.
Thus we have 
$$
    |N_{\tau_n k} - M_k| \leq \sum_{j=1}^{N_T} \mathbbm{1}\{ T_j \in [\tau_n k - \varepsilon, \tau_n k + \varepsilon] \}
$$

Recall that $T_j |N_T = m$ has the same distribution as $j$-th order statistic $u_{(j)}$ of i.i.d.~uniform rvs $u_1,\dots,u_m \sim U([0,T])$. 
{Before taking the conditional expectation, let us for convenience bound $\varepsilon$ by $(N_T + 1) \varepsilon'$ where
\begin{equation}    \label{eq:def-eps-prime}
 \varepsilon' := \max\Big\{T \Big| \frac{\alpha_n}{\alpha(1+q_n)}-1 \Big|, \tau_n + T\Big|\frac{\alpha_n}{\alpha}-1\Big|\Big\}.   
\end{equation}
}
Hence,
$$
    \E\big(|N_{\tau_n k} - M_k| \big| N_T=m\big) \leq \E \bigg(\sum_{j=1}^{m} \mathbbm{1}\{u_j \in [\tau_n k - (m+1)\varepsilon', \tau_n k + (m+1) \varepsilon']\}\bigg) = \frac{2\varepsilon' m (m+1)}{T} .%= 2\tau_n m \max\{1/\alpha, m/T\}
$$

Now we can finally derive the upper bound for $\E \|Y_{T_{N_{\tau_n k}}} - Y_{T_{M_k}}\|$.
By construction we have for any $\ell,  \ell^\prime > 0$ that  $\|Y_{T_\ell} - Y_{T_{\ell^\prime}}\| \leq |\ell-\ell^\prime| \max\|\alpha \nabla V[\mu_{T_j}](T_j)\|\leq \alpha C|\ell-\ell'|$ for some $C>0$. Finally, using the law of total expectation  we get
\begin{align*}
    \E \|Y_{T_{N_{\tau_n k}}} - Y_{T_{M_k}}\| 
    &\leq
    \sum_{m=1}^\infty \alpha \max_{\ell \leq m} \| \nabla V[\mu_{T_\ell}](T_\ell) \| \E \big(|N_{\tau_n k} - M_k| \big| N_T = m\big) \P(N_T = m)
    \\& \leq
   \alpha C \sum_{m=1}^\infty \frac{2 \varepsilon' m(m+1)}{T} \cdot e^{-T/\alpha}\frac{(T/\alpha)^m}{m!}
    \\& \leq C \varepsilon',
\end{align*}
{where $C$ is redefined on the last line. One can verify by inspection that $C$ is independent of~$k$ and $n$. Furthermore, from its definition \eqref{eq:def-eps-prime}, we see that $\varepsilon'$ is bounded by a constant times $\tau_n + |\alpha_n - \alpha|$, which concludes the proof.}
\qedhere
\end{proof}
\newpage
\appendix 

\section{Experimental details}

All experiments are coded in JAX \citep{jax2018github}. Our code is provided at \url{https://github.com/PierreMarion23/dropout-phase-diagram}.

\subsection{Teacher-student experiment} \label{subsec:exp-details-trajs}

We train a two-layer ReLU neural networks with $d$-dimensional input and $n$ neurons. The neurons are initialized with i.i.d.~Gaussians with variance $1/d$ on the inner layer and variance~$1$ on the outer layer. The input data is standard Gaussian, while the target is generated from a teacher network with the same architecture. The network is trained with (full-batch) gradient descent on the mean squared error. We use the feature learning scaling, meaning that the output of the network is multiplied by $1/n$, while  the learning rate is $n \eta_0$, where $\eta_0$ is a base learning rate. The numerical values of the parameters are given in Table \ref{tab:parameters-teacher-student}.

\begin{table}[b]
    \centering
    \begin{tabular}{cc}
        \toprule
       Parameter & Value \\
       \midrule
       Data dimension $d$ & 20 \\
       Width $n$ & $\{200; 500; 1,\!000; 2,\!000; 5,\!000\}$  \\
       Number of gradient steps & $1,\!000$  \\
       Dropout probability $p$ & $0.3$ \\
       Training set size & $500$ \\
       Width of the teacher & $15$ \\
       Base learning rate $\eta_0$ & $0.5$ \\
       \bottomrule
    \end{tabular}
    \caption{Parameters for the teacher-student experiment.}
    \label{tab:parameters-teacher-student}
\end{table}

\begin{figure}[!b]
\centering
\includegraphics[scale=0.48]{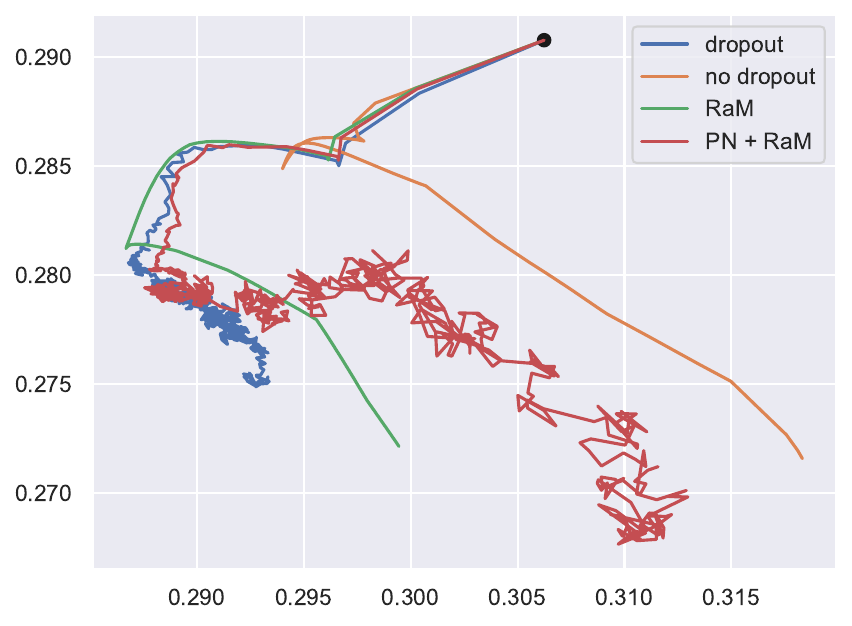}
\hspace{0.5cm}
\includegraphics[scale=0.48]{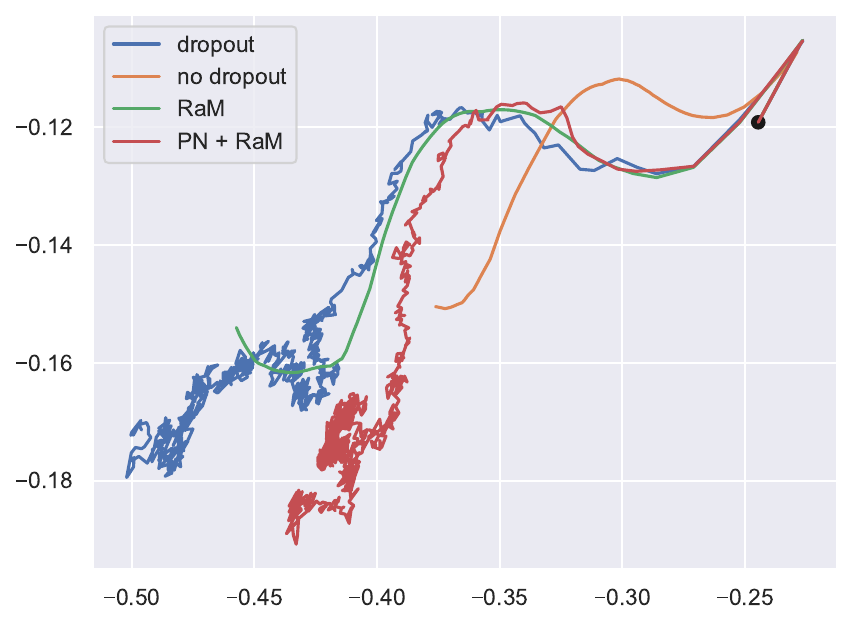}
\caption{Illustration of the pathwise convergence between random metric (RaM) and dropout dynamics (Proposition~\ref{prop:ram-equivalent}) for the teacher-student experiment. Same realization as in Figure~\ref{fig:dropout-ram}, adding GD with propagation noise and random geometry (PN + RaM).}\label{fig:dropout-ram-with-forback}
\end{figure}

We compare four optimization algorithms: plain gradient descent (GD), GD with dropout, GD with random geometry (RaM), and GD with propagation noise and random geometry (PN + RaM), see Section \ref{sec:numerics} for details. For the first three of these algorithms, Figure~\ref{fig:dropout-ram} in the main text shows a realization of the weight trajectory for two arbitrary weights. For completeness, we show in Figure \ref{fig:dropout-ram-with-forback} the same realization for the four algorithms.

\begin{figure}[t]
\centering
\begin{subfigure}{0.45\linewidth}{}
\centering
\includegraphics[scale=0.45]{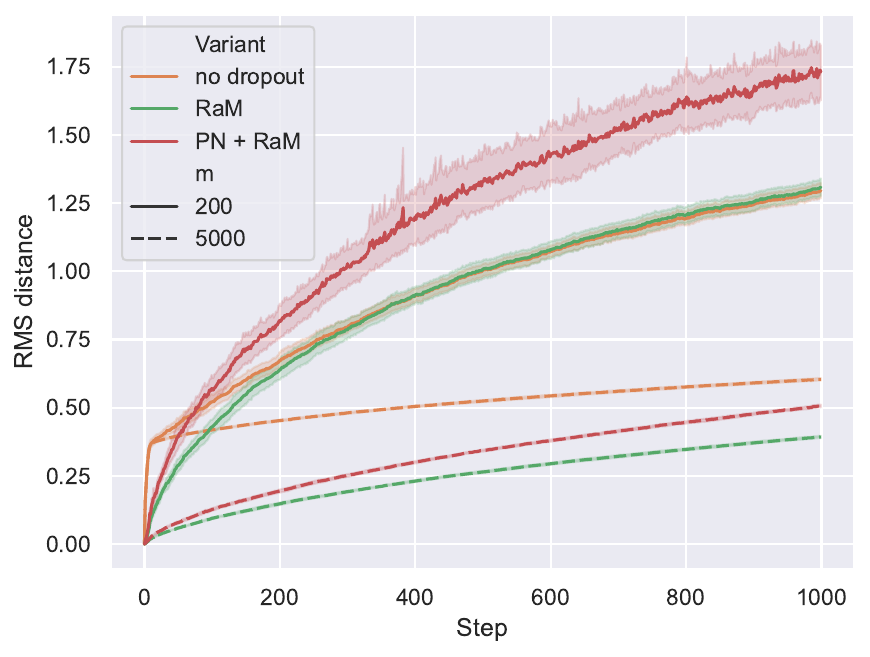}
\caption{As a function of the training time, for two widths.}
\label{fig:dropout-distance-step}
\end{subfigure}%
\hfill
\begin{subfigure}{0.45\linewidth}{}
\centering
\includegraphics[scale=0.45]{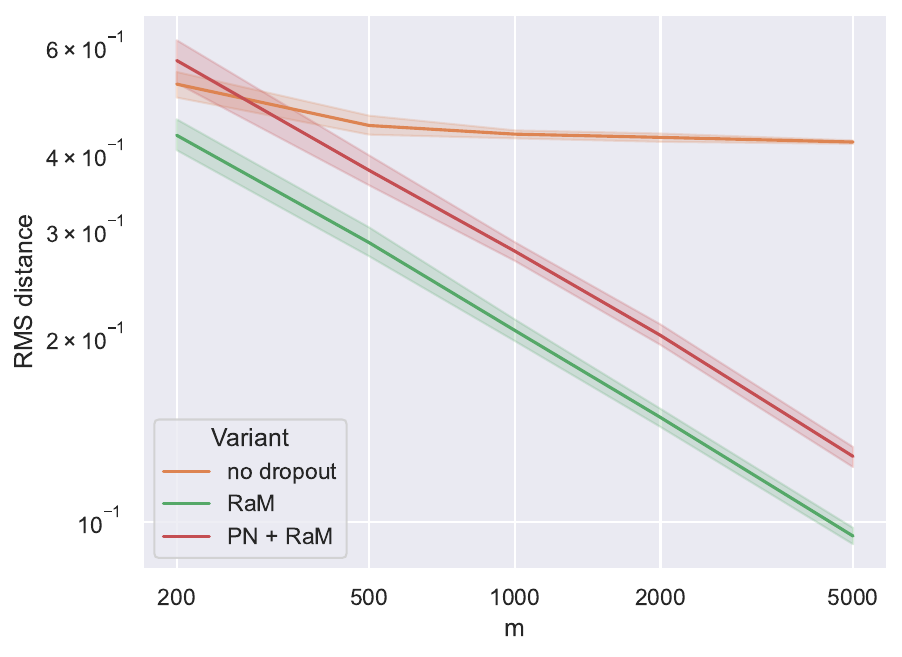}
\caption{As a function of the width, at a fixed training time (after 100 steps).}
\label{fig:dropout-distance-m}
\end{subfigure}%
\caption{RMS distance in parameter space between the dropout algorithm and the three other tested algorithms (plain GD, GD with random geometry and GD with propagation noise and random geometry) for the teacher-student experiment.}\label{fig:dropout-distance}
\end{figure}

\begin{figure}[t]
\centering
\begin{subfigure}{0.45\linewidth}{}
\centering
\includegraphics[scale=0.47]{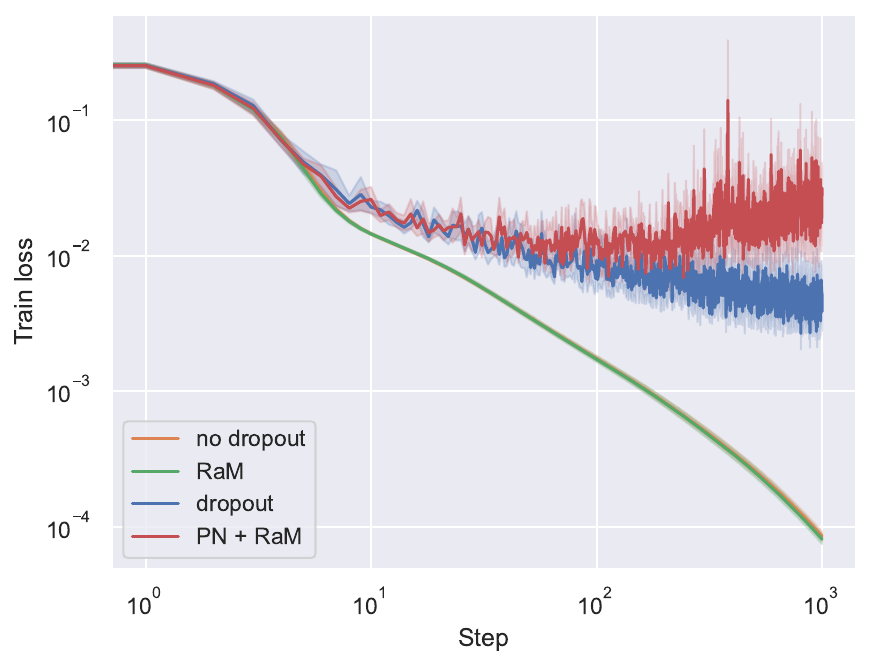}
\caption{Train loss for width $200$.}
\end{subfigure}%
\hfill
\begin{subfigure}{0.45\linewidth}{}
\centering
\includegraphics[scale=0.47]{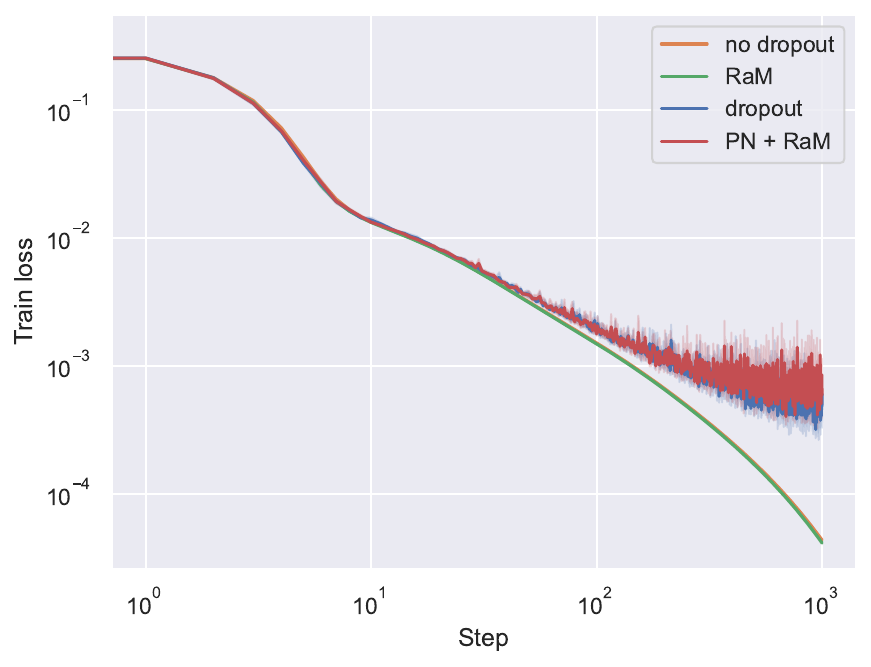}
\caption{Test loss for width $200$.}
\end{subfigure}%
\hfill
\begin{subfigure}{0.45\linewidth}{}
\centering
\includegraphics[scale=0.47]{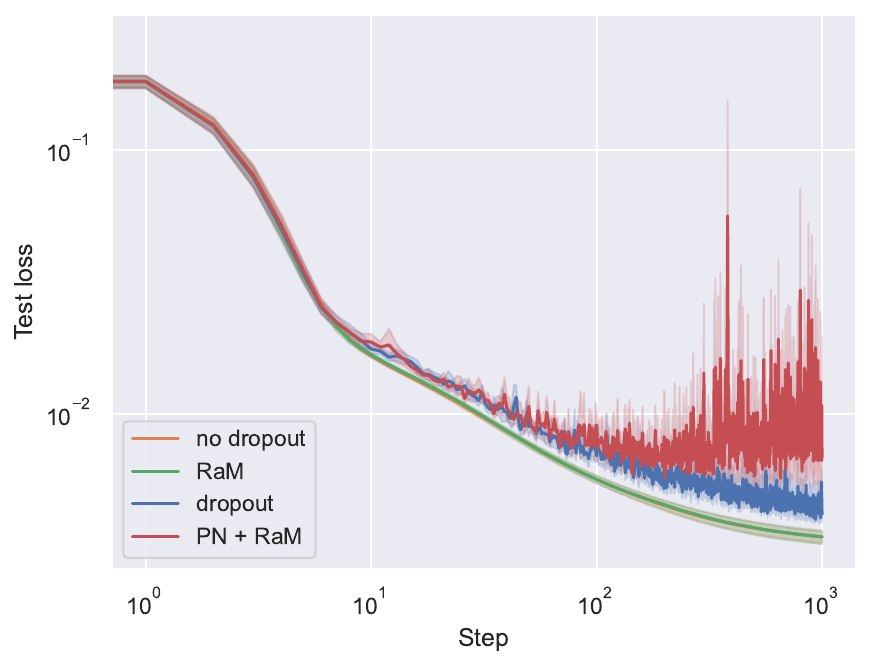}
\caption{Train loss for width $5,\!000$.}
\end{subfigure}%
\hfill
\begin{subfigure}{0.45\linewidth}{}
\centering
\includegraphics[scale=0.47]{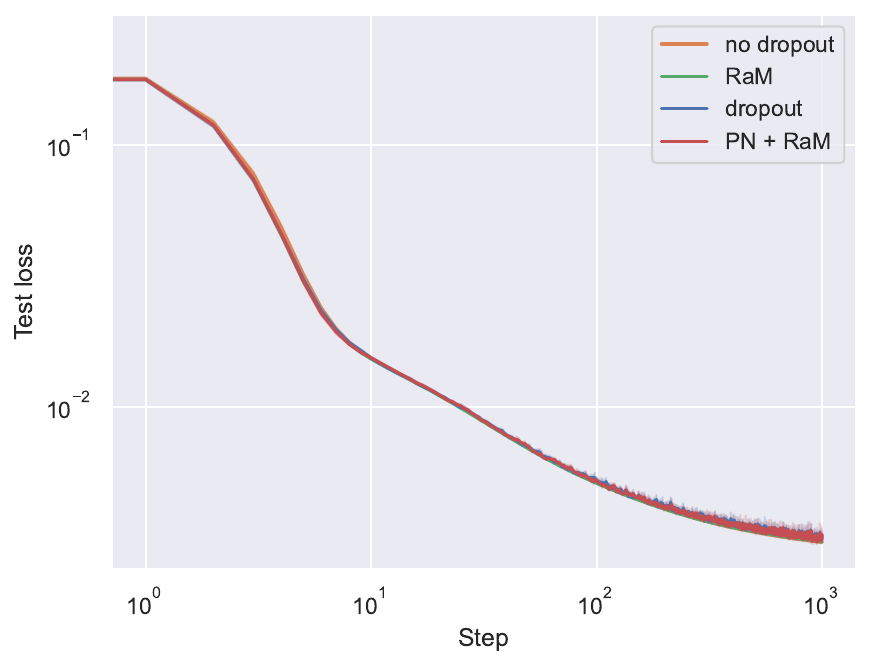}
\caption{Test loss for width $5,\!000$.}
\end{subfigure}%
\caption{Evolution of losses during training for the teacher-student experiment with the four tested optimization algorithms.}\label{fig:losses}
\end{figure}

To provide metrics for the comparison between these algorithms, we compute the RMS distance ($\ell_2$ distance normalized by $\sqrt{n}$) in parameter space between the dropout version and each of the other three versions. The results are presented in Figure \ref{fig:dropout-distance-step} as a function of the number of gradient steps for two different widths, and in Figure \ref{fig:dropout-distance-m} as a function of the width for a fixed number of steps. Confidence intervals over $10$ repetitions are plotted. We observe that dropout is closer to RaM than to plain GD and PN+RaM, especially for large width and small training time, which is expected from our theory (Proposition \ref{prop:ram-equivalent}).

We observe in Figure \ref{fig:losses} that, in terms of performance metrics, in this setting, RaM and plain GD both outperform dropout, and perform very similarly to each other. In other words, in this setting the propagation noise from the forward pass is harmful for performance, even more so for the PN+RaM algorithm.

The experiment takes about 30 minutes to run on a consumer laptop.

\subsection{MNIST experiment} \label{subsec:exp-details-mnist}

We train a two-layer ReLU network on the binary classification of digits 4 and 7 of MNIST. The images are flattened into vectors of dimension $784$, and pixel values are mapped into $[0, 1]$. The loss is the logistic loss. We use the standard data split from MNIST, removing at random $2,\!000$ samples from the training set to serve as validation set. The initialization and scaling of the neural network are as described in Section \ref{subsec:exp-details-trajs}.

We compare five optimization algorithms: plain mini-batch stochastic gradient descent (SGD, no dropout), SGD with dropout, SGD with random geometry (RaM), SGD with propagation noise (PN), SGD with propagation noise and random geometry (PN + RaM), see Section \ref{sec:numerics} for details. The numerical values of the parameters are given in Table \ref{tab:parameters-mnist}. 

Confidence intervals over $10$ repetitions are plotted. In Figure \ref{fig:best-test-loss-mnist}, we log the loss value over the course of training every $100$ steps, and report the minimum logged test loss. 

\begin{table}[t]
    \centering
    \begin{tabular}{cc}
        \toprule
       Parameter & Value \\
       \midrule
       Width $n$ & $4,\!000$  \\
       Number of gradient steps & $40,\!000$  \\
       Dropout probability $p$ & $\{0.1, 0.2, 0.3, 0.4\}$ \\
       Base learning rate $\eta_0$ & $0.25$ \\
       Batch size & $64$ \\
       \bottomrule
    \end{tabular}
    \caption{Parameters for the MNIST experiment.}
    \label{tab:parameters-mnist}
\end{table}

The experiment takes about 3 hours to run on an Nvidia GTX 1080Ti GPU.

\section*{Acknowledgments}
P.M.~is supported by a Google PhD Fellowship. This work was done in part while P.M.~was visiting the Simons Institute for the Theory of Computing.

\newpage

\bibliography{LC.bib}

\end{document}

%% file: shortcuts.tex
\newcommand{\RR}{\mathbb{R}}
\newcommand{\NN}{\mathbb{N}}

\newcommand{\EE}{\mathbb{E}}

\newcommand{\Cc}{\mathcal{C}}
\newcommand{\Dd}{\mathcal{D}}

\newcommand{\Ff}{\mathcal{F}}

\newcommand{\Ll}{\mathcal{L}}

\newcommand{\Pp}{\mathcal{P}}

\newcommand{\Ss}{\mathcal{S}}

\newcommand{\Xx}{\mathcal{X}}

\newcommand{\E}{\mathbf{E}}
\renewcommand{\P}{\mathbf{P}}

\renewcommand{\d}{\mathrm{d}}

\newcommand{\Id}{\mathrm{Id}}

\newcommand{\h}{\mathfrak{h}}

\newcommand{\Lip}{\mathrm{Lip}}

\DeclareMathOperator{\dist}{dist}

\newtheorem{theorem}{Theorem}
\newtheorem{proposition}[theorem]{Proposition}
\newtheorem{lemma}[theorem]{Lemma}
\newtheorem{corollary}[theorem]{Corollary}

\newtheorem{assumption}{Assumption}

%% file: phase-diagram.tex
% --------------------------------------------------------
% PHASE DIAGRAM  (n, τ_n, q_n)  – refined version
% --------------------------------------------------------
\begin{tikzpicture}[scale=1.55, thick, >=latex]

% ---------- vertices of an equilateral triangle ----------
\pgfmathsetmacro{\h}{2*sqrt(3)}          % height for side length 4
\coordinate (A) at ( 0,  \h);            % top      :  n → ∞
\coordinate (L) at (-2,  0);             % left     : 1/τ_n → ∞
\coordinate (R) at ( 2,  0);             % right    : 1/q_n → ∞

% ---------- useful internal points ----------
\coordinate (B) at ($(L)!0.5!(R)$);      % midpoint of base
\coordinate (BB) at ($(A)!0.5!(R)$);      % midpoint of base
\coordinate (M) at ($(A)!0.5!(R)$);      % midpoint of edge A–R
\coordinate (C) at ($(A)!2/3!(B)$);      % centroid (intersection of medians)

% ---------- background colouring (clip to triangle) -------
\begin{scope}
  \clip (A) -- (R) -- (L) -- cycle;

  % red (degenerate) region – right of vertical median OR below left median
  \fill[red!25]   (A) -- (B) -- (R) -- cycle;
  \fill[red!25]   (L) -- (C) -- (B) -- cycle;

  % grey (non-degenerate) region – remaining top-left wedge
  \fill[gray!50]  (A) -- (L) -- (C) -- cycle;
\end{scope}

% ---------- triangle outline ----------
\draw (A) -- (R) -- (L) -- cycle;

% ---------- highlighted medians ----------
% vertical median (blue)
\draw[blue,line width=1.5mm]   (A) -- (C);
% left median (green)
\draw[ green!70!black, line width=1.5mm] (L) -- (C);

% optional dashed helper axes from C
\draw[dashed] (C) -- (B);                % downwards (α–axis)
\draw[dashed] (C) -- (BB);    % into red wedge (β–axis)

% ---------- special points ----------
% top vertex (orange) & centroid (magenta)
\fill[orange]  (A) circle (.095);
\fill[purple]  (C) circle (.095);

% ---------- labels for vertices ----------
\node[above=2pt]              at (A) {\footnotesize{$\log n$ dominates}};
\node[below left=2pt]         at (L) {\footnotesize{$\log \tau_n^{-1} $ dominates}};
\node[below right=2pt]        at (R) {\footnotesize{$\log{q_n^{-1}}$ dominates}};

% ---------- α–axis ticks & labels along base -----------------
\node[below=7pt] at ($(L)!0.25!(R)$)         {\footnotesize{$\alpha=0$}};
\node[below=7pt] at ($(L)!0.5!(R)$) {\footnotesize{$\alpha\in\mathbb R^{\!*}_{+}$}};
\node[below=7pt] at ($(L)!0.75!(R)$)         {\footnotesize{$\alpha=+\infty$}};

% ---------- β–axis ticks & labels along right side -----------
\node[above right=4pt]  at ($(A)!0.25!(R)$) {\footnotesize{$\beta=0$}};
\node[above right=0pt]  at ($(A)!0.5!(R)$) {\footnotesize{$\beta\in\mathbb R^{\!*}_{+}$}};
\node[above right=4pt]  at ($(A)!0.75!(R)$) {\footnotesize{$\beta=+\infty$}};

% ---------- text annotations (colours match arrows/lines) ----
%\node[orange,  above right=3pt of A]  {discrete-time jump process};
%\node[blue,    right=6pt  of $(A)!0.45!(C)$, align=left] 
%      {continuous-time\\ jump process};
%\node[purple,  left=5pt  of $(C)!0.55!(A)$] {critical\\ limit};
%\node[green!70!black, left=3pt of $(L)!0.55!(C)$] {penalised WGF};
%\node[black,  left=6pt  of $(L)!0.3!(A)$] {WGF};
%\node[red!70!black, right=4pt of $(C)!0.4!(R)$] {degenerate};

% ---------- braces (optional visual aids for β range) --------
\draw[decorate, decoration={brace, amplitude=10pt}, thin]
      ($(A)!0.02!(R)$) -- ($(A)!0.48!(R)$);
\draw[decorate, decoration={brace, amplitude=10pt}, thin]
      ($(A)!0.52!(R)$) -- ($(A)!0.98!(R)$);
\draw[decorate, decoration={brace, amplitude=10pt}, thin]
      ($(R)!0.02!(L)$) -- ($(R)!0.48!(L)$);
\draw[decorate, decoration={brace, amplitude=10pt}, thin]
      ($(R)!0.52!(L)$) -- ($(R)!0.98!(L)$);

\end{tikzpicture}